\documentclass[sn-mathphys-num]{sn-jnl}

\usepackage[utf8]{inputenc}
\usepackage{graphicx}
\usepackage{diagbox} 
\usepackage{hyperref}
\expandafter\let\csname equation*\endcsname\relax
\expandafter\let\csname endequation*\endcsname\relax
\usepackage{amsmath}
\usepackage{amsfonts}
\usepackage{amssymb} 
\usepackage{amsthm}
\usepackage{url}
\usepackage{txfonts}
\usepackage{url}
\usepackage{dsfont}
\usepackage{nicefrac}       % compact symbols for 1/2, etc.
\usepackage{xfrac}
\usepackage[USenglish]{babel}
\usepackage{multirow}
\usepackage{paralist}
\usepackage{booktabs}
\usepackage{microtype}      % microtypography
\usepackage{xcolor,colortbl}
\usepackage{hhline}
\usepackage{bbm}
\usepackage{tabularx}
\usepackage{lscape}
\usepackage{multirow,makecell}
\usepackage{color,soul}
\usepackage{algorithm}
\usepackage{algorithmicx}
\usepackage{algpseudocode}
\usepackage{xspace}
\usepackage{wrapfig}
\usepackage{harpoon}
\usepackage{mathrsfs}
\usepackage{bm}

\usepackage{geometry}
\usepackage{comment}

\def\S{{\bf S}}

\def\E{\mathbb{E}}
\usepackage[caption=false]{subfig}

\usepackage{pgfplots}

\usepackage{amsthm}
\theoremstyle{thmstyleone}%
\newtheorem{theorem}{Theorem}[section]%  meant for continuous numbers
\theoremstyle{thmstyletwo}%
\newtheorem{example}{Example}%
\newtheorem{assumption}[theorem]{Assumption}
\newtheorem{proposition}[theorem]{Proposition}
\theoremstyle{thmstylethree}%
\newtheorem{definition}[theorem]{Definition}%

\usepackage[final]{changes}
\usepackage{graphicx} % Required for inserting images
\usepackage{amsmath}

\def\E{\mathrm{E}}

\begin{document}

\title[Wasserstein Uncertain Model Reconstruction]{A local squared Wasserstein-2 method for efficient reconstruction of models with uncertainty}

\author*[1]{\fnm{Mingtao} \sur{Xia}}\email{xiamingtao@nyu.edu}

\author[2]{\fnm{Qijing} \sur{Shen}}\email{qijing.shen@ndm.ox.ac.uk}

\affil[1]{\orgdiv{Courant Institute of Mathematical Sciences}, \orgname{New York University}, \orgaddress{\street{251 Mercer Street}, \city{New York}, \postcode{10012}, \state{NY}, \country{USA}}}

\affil[2]{\orgdiv{Nuffield Department of Medicine}, \orgname{University of Oxford}, \orgaddress{ \city{Oxford}, \postcode{OX3 7BN}, \country{UK}}}

%%==================================%%
%% Sample for unstructured abstract %%
%%==================================%%

\abstract{In this paper, we propose a local squared Wasserstein-2 ($W_2$) method to solve the inverse problem of reconstructing models with uncertain latent variables or parameters. A key advantage of our approach is that it does not require prior information on the distribution of the latent variables or parameters in the underlying models. Instead, our method can efficiently reconstruct the distributions of the output associated with different inputs based on empirical distributions of observation data.
We demonstrate the effectiveness of our proposed method across several uncertainty quantification (UQ) tasks, including linear regression with coefficient uncertainty, training neural networks with weight uncertainty, and reconstructing ordinary differential equations (ODEs) with a latent random variable.}


%%\pacs[JEL Classification]{D8, H51}

%%\pacs[MSC Classification]{35A01, 65L10, 65L12, 65L20, 65L70}

\maketitle

\section{Introduction}
Models incorporating uncertainty have been extensively utilized across various fields. 
For example, models incorporating measurement errors are widely used \cite{fuller2009measurement, carroll1995measurement, schennach2004estimation}. Additionally, models involving latent unobserved variables are frequently employed 
in uncertainty quantification (UQ) \cite{bishop1998latent}, with applications in stock price modeling \cite{yang2020parameter} and image processing \cite{aliakbarpour2016heterogeneous}.In bioinformatics, when analyzing people's propensity to get infected by certain types of genotype-influenced diseases, dimension reduction techniques are often employed to eliminate genes with minor relevance to the disease \cite{ansari2017genome}.
In these models, instead of offering a single deterministic output, the output is sampled from a distribution influenced by the input.

The reconstruction of models with uncertainty from data has received significant research interest \cite{cheng2023machine,zhang2012bayesian}.
Traditional methods for reconstructing models with uncertainty primarily focus on parameter inferences.
These approaches typically start by assuming a specific model form with several unknown parameters and then aim to infer the mean and variance of these parameters from the data \cite{kuczera1998monte, arcieri2023bridging}.
 Recent advancements in Bayesian methods, especially the Bayesian neural network (BNN) \cite{goan2020bayesian, shang2023approximating}, make it possible to learn the posterior distribution of unknown and uncertain model parameters given their prior distributions as well as observed data. 

The Wasserstein distance, which measures the discrepancy between two probability distributions \cite{villani2009optimal, zheng2020nonparametric}, has recently become a popular research topic in UQ \cite{kidger2021neural}. For example, regularized Wasserstein distance methods have been proposed for multi-label prediction problems \cite{frogner2015learning} and imaging applications \cite{adler2017learning}. Additionally, the Wasserstein generative adversarial network (WGAN) \cite{gao2022wasserstein} has been applied to various tasks, such as image generation \cite{jin2019image, wang2023improved} and generating the distribution of solutions to partial differential equations with latent parameters \cite{gao2022wasserstein}. However, training a generative adversarial network model can be challenging and computationally expensive \cite{saxena2021generative}.

In this work, we study the following model with uncertainty:
\begin{equation}
    \bm{y}(\bm{x};\omega) \coloneqq \bm{f}(\bm{x},\omega),\,\, \bm{x}\in D\subseteq\mathbb{R}^n
    \label{model_object}
\end{equation}
where $\bm{f}(\cdot;\cdot):\mathbb{R}^n\times \Omega\rightarrow\mathbb{R}^d$ is a continuous function in $\bm{x}$; 
$\omega\in\Omega$ is a latent random variable in a sample space $\Omega$. Only $\bm{x}$ is observed (referred to as the input). Therefore, $\bm{y}(\bm{x};\omega)$ follows a distribution determined by $\bm{x}$. 
Our goal is to reconstruct a model:
\begin{equation}
    \hat{\bm{y}}(\bm{x}; \hat{\omega}) \coloneqq \hat{\bm{f}}(\bm{x},\hat{\omega}), \,\, \bm{x}\in D\subseteq\mathbb{R}^n
    \label{approximate}
\end{equation}
as an approximation to Eq.~\eqref{model_object} in the sense that the distribution of $\bm{y}(\bm{x};\omega)$ can be matched by the distribution of $\hat{\bm{y}}(\bm{x};\hat{\omega})$ for the same input $\bm{x}$. In Eq.~\eqref{approximate}, $\hat{\omega}\in\hat{\Omega}$ is another random variable in another sample space $\hat{\Omega}$ (we do not require $\hat{\Omega}$ to be the same as $\Omega$). To our knowledge, there exist few methods that directly reconstruct the distribution of $\bm{y}(\bm{x}; \omega)$ for different $\bm{x}$ in Eq.~\eqref{model_object} without requiring a specific form of $\bm{f}$ or a prior distribution of $\omega$.

In our paper, we propose and analyze a novel local squared Wasserstein-2 ($W_2$)
method to reconstruct a model Eq.~\eqref{approximate} for approximating the uncertainty model Eq.~\eqref{model_object}. 
Our main contributions are: i) we propose and analyze a local squared $W_2$ loss function for reconstructing uncertainty models in UQ, which could be efficiently evaluated by empirical distributions from a finite number of observations
, ii) unlike the Bayesian methods or previous Wasserstein-distance-based methods \cite{yang58imprecise}, our method does not require a prior distribution of $\omega$ nor does it necessarily require an explicit form of $\bm{f}$ in Eq.~\eqref{model_object}, and iii) through numerical experiments, we showcase the efficacy of our proposed method in different UQ tasks such as linear regression with coefficient uncertainty, training a neural network with weight uncertainty, and reconstructing an ODE with latent uncertain parameters. 

\section{Results}
\label{numerical_result}
\subsection{Local squared $W_2$ loss function}
\label{local_definition}
We present a novel local squared $W_2$ loss function:
\begin{equation}
      \tilde{W}_{2, \delta}^{2, \text{e}}(\bm{y}, \hat{\bm{y}})\coloneqq\int_D W_2^2(\mu_{\bm{x}, \delta}^{\text{e}}, \hat{\mu}_{\bm{x}, \delta}^{\text{e}})\nu^{\text{e}}(\text{d}\bm{x}) 
    \label{localw2}
\end{equation}
which approximates the quantity
\begin{equation}
\tilde{W}_{2}^{2}(\bm{y}, \hat{\bm{y}})\coloneqq\int_D W_2^2(\mu_{\bm{x}}, \hat{\mu}_{\bm{x}})\nu(\text{d}\bm{x}).
\label{approx_quantity}
\end{equation}
In Eqs.~\eqref{localw2} and \eqref{approx_quantity}, $\nu(\cdot)$ and $\nu^{\text{e}}(\cdot)$ are the distribution and the empirical distribution of $\bm{x}$, respectively. $\bm{y}, \hat{\bm{y}}$ correspond to the LHS of ground truth model Eq.~\eqref{model_object} and the LHS of the approximate model Eq.~\eqref{approximate}, respectively. $W_2^2$ is the squared $W_2$ distance (detailed definition given in Definition~\ref{def:W2}).
In Eqs.~\eqref{localw2} and ~\eqref{approx_quantity}, $\mu_{\bm{x}}$ is the distribution of $\bm{y}(\bm{x}; \omega)$ when $\bm{x}$ is fixed, and $\mu_{\bm{x}, \delta}^{\text{e}}$ is the empirical conditional distribution of $\bm{y}(\tilde{\bm{x}};\omega)$ conditioned on  
$|\tilde{\bm{x}} - \bm{x}|_{x}\leq \delta$. Similarly, $\hat{\mu}_{\bm{x}}$ is the distribution of $\hat{\bm{y}}(\bm{x}, \hat{\omega})$ when $\bm{x}$ is fixed, and $\hat{\mu}_{\bm{x}, \delta}^{\text{e}}$ is the empirical conditional distribution of $\hat{\bm{y}}(\tilde{\bm{x}};\hat{\omega})$ conditioned on $|\tilde{\bm{x}} - \bm{x}|_{x}\leq \delta$, respectively. $|\cdot|_{x}$ denotes a norm for $\bm{x}\in\mathbb{R}^d$. 

Our local squared $W_2$ method approximates Eq.~\eqref{model_object} using  Eq.~\eqref{approximate} by minimizing the local squared $W_2$ loss function $\tilde{W}_{2, \delta}^{2, \text{e}}(\bm{y}, \hat{\bm{y}})$ in Eq.~\eqref{localw2}. 
Analysis on why minimizing $\tilde{W}_{2, \delta}^{2, \text{e}}(\bm{y}, \hat{\bm{y}})$, as an approximation to Eq.~\eqref{approx_quantity}, leads to the successful reconstruction of Eq.~\eqref{model_object} is in Subsection~\ref{W2_loss}.
We shall test the effectiveness of our local squared $W_2$ method across several different UQ tasks. In this paper, $\|\cdot\|$ refers to the $l^2$ norm of a vector and the errors in the mean $\E[\hat{y}]$ and the standard deviation $\text{SD}[\hat{y}]$ stand for the relative errors:
\begin{equation}
    \text{error in~}\E[\hat{y}]\coloneqq\frac{\int_D\Big|\E[y(\bm{x};\omega)] - \E[\hat{y}(\bm{x};\hat{\omega})]\Big|\nu^{\text{e}}(\text{d}\bm{x})}{\int_D|\E[y(x_i;\omega)]]|\nu^{\text{e}}(\text{d}\bm{x})},\,\,\,\,
    \text{error in SD~}[\hat{y}]\coloneqq\frac{\int_D\Big|\text{SD}[y(x_i;\omega)] - \text{SD}[\hat{y}(x_i;\hat{\omega})]\Big|\nu^{\text{e}}(\text{d}\bm{x})}{\int_D|\text{SD}[y(x_i;\omega)]|\nu^{\text{e}}(\text{d}\bm{x})}.
\end{equation}
%where $\nu^{\text{e, t}}$ is the empirical distribution of $\bm{x}$ in the testing set.

\subsection{Linear regression with coefficient uncertainty}
\label{example2}
We first apply our proposed local squared $W_2$ method to a linear regression problem with coefficient uncertainty. 
We consider the following linear model whose coefficients are sampled from the normal distribution
\cite{raftery1993model}:
    \begin{eqnarray}
        y(\bm{x}; \omega) = \sum_{i=1}^3\omega_i x_i + \omega_0,\,\, \omega_i\sim\mathcal{N}(b_i, \sigma_i^2).
        \label{example2_model}
    \end{eqnarray}
    We assume that $\bm{x}$ is independent of $\omega$ and $\omega_i$ is independent of $\omega_j$ when $i\neq j$. In Eq.~\eqref{example2_model}, we set the ground truth:
    \begin{equation}
        (b_0, b_1, b_2, b_3)=(1,1,2,3), \,\,\,\, (\sigma_0, \sigma_1, \sigma_2, \sigma_3)=(0.1, 0.2, 0.3, 0.4).
        \label{w_ground_truth}
    \end{equation}
     We aim to develop another linear model:
\begin{equation}
    \hat{y}(\bm{x};\hat{\omega}) = \sum_{i=1}^3\hat{\omega}_i x_i + \hat{\omega}_0, \,\,\,\, \hat{\omega}_i\sim \mathcal{N}(\hat{b}_i, \hat{\sigma}_i^2), \,\,\hat{\bm{b}}\coloneqq(\hat{b}_1,\hat{b}_1,\hat{b}_1,\hat{b}_1),\,\,\hat{\bm{\sigma}} \coloneqq (\hat{\sigma}_1,\hat{\sigma}_2, \hat{\sigma}_3, \hat{\sigma}_4)
        \label{example2_approx}
\end{equation}
to approximate Eq.~\eqref{example2_model} so that the distribution of $y(\bm{x};\omega)$ can be matched by the distribution of $\hat{y}(\bm{x};\hat{\omega})$ when fixing $\bm{x}$.
In Eq.~\eqref{example2_approx}, we assume that $\bm{x}$ is independent of $\hat{\omega}$ and $\hat{\omega}_i$ is independent of $\hat{\omega}_j$ when $i\neq j$. 
For the training data $\{(\bm{x}_i, y_i)\}_{i=1}^N$, we let $x_1, x_2, x_3$ be independent of each other and sample $\bm{x}\coloneqq(x_1,x_2, x_3)$ from the following distributions:
    \begin{eqnarray}
        x_1\sim \text{Exp}(4), \,\, x_2\sim\mathcal{N}(0, 0.25),\,\, x_3\sim \text{Be}(5, 5).
    \end{eqnarray}
$\text{Exp}(4)$ denotes the exponential distribution with intensity parameter $4$, while $\text{Be}(5, 5)$ represents the Beta distribution with both its shape and scale parameters set to $5$.

We minimize the local squared $W_2$ distance Eq.~\eqref{localw2} in order to obtain $\hat{b}_i$ and $\hat{\sigma}_i$ in Eq.~\eqref{example2_approx}.
When determining the neighborhood $|\tilde{\bm{x}}-\bm{x}|\leq\delta$ of $\bm{x}$ for evaluating the empirical distributions $\mu_{\bm{x}, \delta}^{\text{e}}$ and $\hat{\mu}_{\bm{x}, \delta}^{\text{e}}$ in Eq.~\eqref{localw2}, two different norms of $\bm{x}$ are used:
\begin{equation}
|\bm{x}|_{\text{homo}}^2\coloneqq\sum_{i=1}^3 x_i^2,\,\,|\bm{x}|^2_{\text{hete}}\coloneqq\sum_{i=1}^n c_i^2x_i^2,
\label{norm_def}
\end{equation}
where $c_i$ in $|\bm{x}|^2_{\text{hete}}$ are obtained from carrying out a linear regression of $y$ w.r.t. $\bm{x}$ by minimizing:
\begin{equation}
\sum_{i=1}^N \Big(y_i - \sum_{i=1}^3 c_ix_i - c_0\Big)^2.
\end{equation}
Using $|\cdot|_{\text{heto}}$ accounts for the heterogeneity in the dependencies of $y$ on $x_1, x_2, x_3$ in Eq.~\eqref{example2_model}. We use the average relative errors to measure errors in the reconstructed $\hat{b}_i$ and $\hat{\sigma}_i, i=1,2,3,4$ in Eq.~\eqref{example2_approx}:
\begin{equation}
    \text{Error~in~} \hat{\bm{b}}\coloneqq \frac{|b_i-\hat{b}_i|}{\sum_{i=0}^4|b_i|},\,\,\,\,\,\,\,\, \text{Error~in~ }\hat{\bm{\sigma}}\coloneqq\frac{\big||\sigma_i|-|\hat{\sigma}_i|\big|}{\sum_{i=0}^4|\sigma_i|}.
    \label{error_metric}
\end{equation}

    \begin{figure}
    \centering
\includegraphics[width=\linewidth]{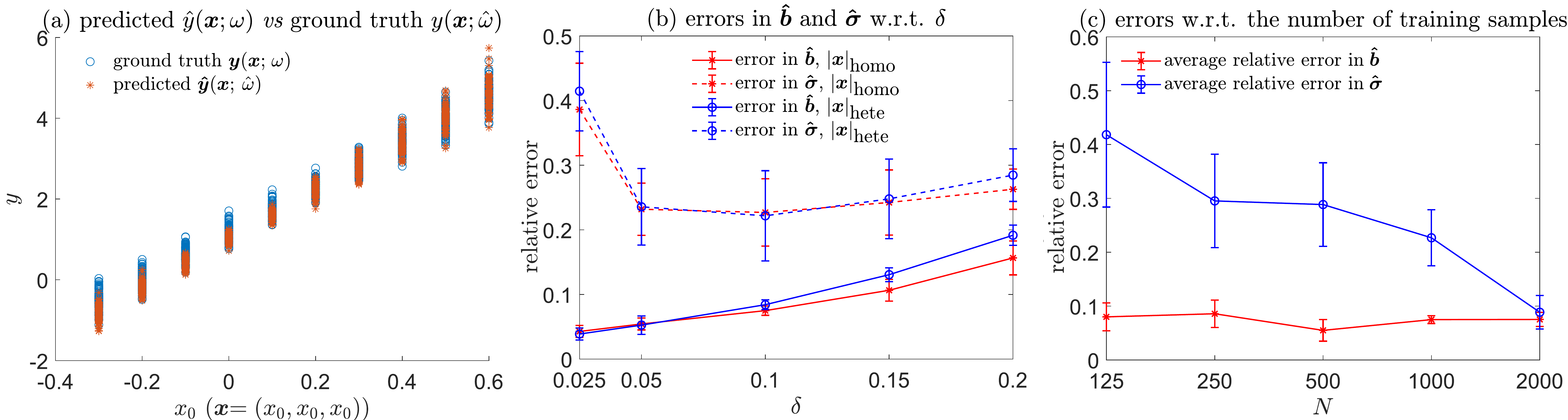}
    \caption{(a) The predicted $\hat{\bm{y}}(\bm{x};\omega)$ versus the ground truth $\bm{y}(\bm{x}, \hat{\omega})$. To illustrate, we take $\bm{x}$ on the line $\bm{x}=(x_0, x_0, x_0)$ and choose different values of $x_0=-0.3+0.1i, i=0,\ldots,9$. At each $\bm{x}$,
    we independently sample 100 $\omega=(\omega_1, \omega_2, \omega_3, \omega_4)$ in Eq.~\eqref{example2_model} as well as $\hat{\omega}=(\hat{\omega}_1, \hat{\omega}_2, \hat{\omega}_3, \hat{\omega}_4)$ in Eq.~\eqref{example2_approx} and plot 100 ground truth $y(\bm{x};\omega)$ versus 100 predicted $\hat{y}(\bm{x};\hat{\omega})$.
    (b) The average relative errors in $\hat{b}_i$ and $\hat{\sigma}_i$ w.r.t. the size of neighborhood $\delta$ when using the two different norms of the input $\bm{x}$: $|\bm{x}|_{\text{homo}}$ and $|\bm{x}|_{\text{hete}}$. (c) The average relative errors in $\hat{b}_i$ and $\hat{\sigma}_i$ w.r.t. the number of training samples $N$. In (c), the norm for $\bm{x}$ is $|\bm{x}|_{\text{hete}}$ (defined in Eq.~\eqref{norm_def}) and the size of neighborhood$\delta=0.1$ . }
    \label{fig:example2}
\end{figure}

In Fig.~\ref{fig:example2} (a), the distribution of the predicted $\hat{\bm{y}}(\bm{x};\hat{\omega})$ matches well with the distribution of the ground truth $\bm{y}(\bm{x};\omega)$ on the line $\bm{x} = (x_0, x_0, x_0)$. In Fig.~\ref{fig:example2} (b), the errors in the reconstructed $\hat{\bm{b}}$ and $\hat{\bm{\sigma}}$ are not sensitive to whether using $|\bm{x}|_{\text{homo}}$ or $|\bm{x}|_{\text{hete}}$. However, when the size of neighborhood $\delta$ in Eq.~\ref{localw2} is too small ($\delta=0.025$), the error in the reconstructed standard deviation $\hat{\bm{\sigma}}$ is large. When $\delta$ is too small, the local squared $W_2$ loss Eq.~\eqref{localw2} might not be a good approximation of $\tilde{W}_{2}^{2}(\bm{y}, \hat{\bm{y}})$ in Eq.~\eqref{localw2}, leading to the poor reconstruction of Eq.~\eqref{example2_model}. On the other hand,
the error in the reconstructed mean $\hat{\bm{b}}$ gets larger when $\delta$ increases. Errors in the reconstructed mean and standard deviation are both maintained small when $\delta\in[0.05, 0.1]$. 
 The error in the reconstructed standard deviation $\hat{\bm{\sigma}}$ decreases as the number of training samples $N$ increases while the error in the reconstructed $\hat{\bm{b}}$ is not very sensitive to $N$ (shown in Fig.~\ref{fig:example2} (c)).
To conclude, our local squared $W_2$ method can accurately reconstruct the linear model Eq.~\eqref{example2_model} with coefficient uncertainty when $\delta\in[0.05, 0.1]$ and sufficient training data is available.

\subsection{Training a neural network model with weight uncertainty}
\label{example1}
Next, we consider reconstructing the following nonlinear uncertainty model \cite{rooney2001design, bates1988nonlinear}:
\begin{equation}
    y(x;\omega)=\omega_1 \big(1- \exp(-\omega_2x)\big) + 5,
    \label{example1_model}
\end{equation}
where $\omega=(\omega_1, \omega_2)^T$ are the latent random variables in the model.
We assume that $x$ and $\omega$ are independent. 
We independently generate 1000 samples for training with $x\sim\mathcal{U}(-0.5, 0.5)$ and 
$(\omega_1, \omega_2)^T \sim\mathcal{N}\big((19.1426, 0.5311)^T, \Sigma\big)$, where $\Sigma$ is the covariance matrix:
\begin{equation}
    \Sigma= 
\begin{bmatrix}
6.22864 & -0.4322 \\
-0.4322 & 0.04124\\
\end{bmatrix}.
\end{equation}

A parameterized neural network model with weight uncertainty in Fig.~\ref{fig:nn_model} is used as $\hat{\bm{f}}$ in Eq.~\eqref{approximate} which approximates Eq.~\eqref{example1_model}. We aim to optimize the mean and variance of weights $\{w_{i,j,k}\}$ as well as the bias $\{b_{i, k}\}$ in the neural network by minimizing Eq.~\eqref{localw2} such that the distribution of $\hat{y}$ aligns with the distribution of $y$ given the same $x$. For testing, we generate a testing set $T=\cup_{i=0}^{11} T_i$ with each $T_i$ containing 100 samples $(x_{r, i}, y(x_{r, i}; \omega)), x_{r, i}=0.1i-0.5, r=1,\ldots,100$.

    \begin{figure}
    \centering
\includegraphics[width=\linewidth]{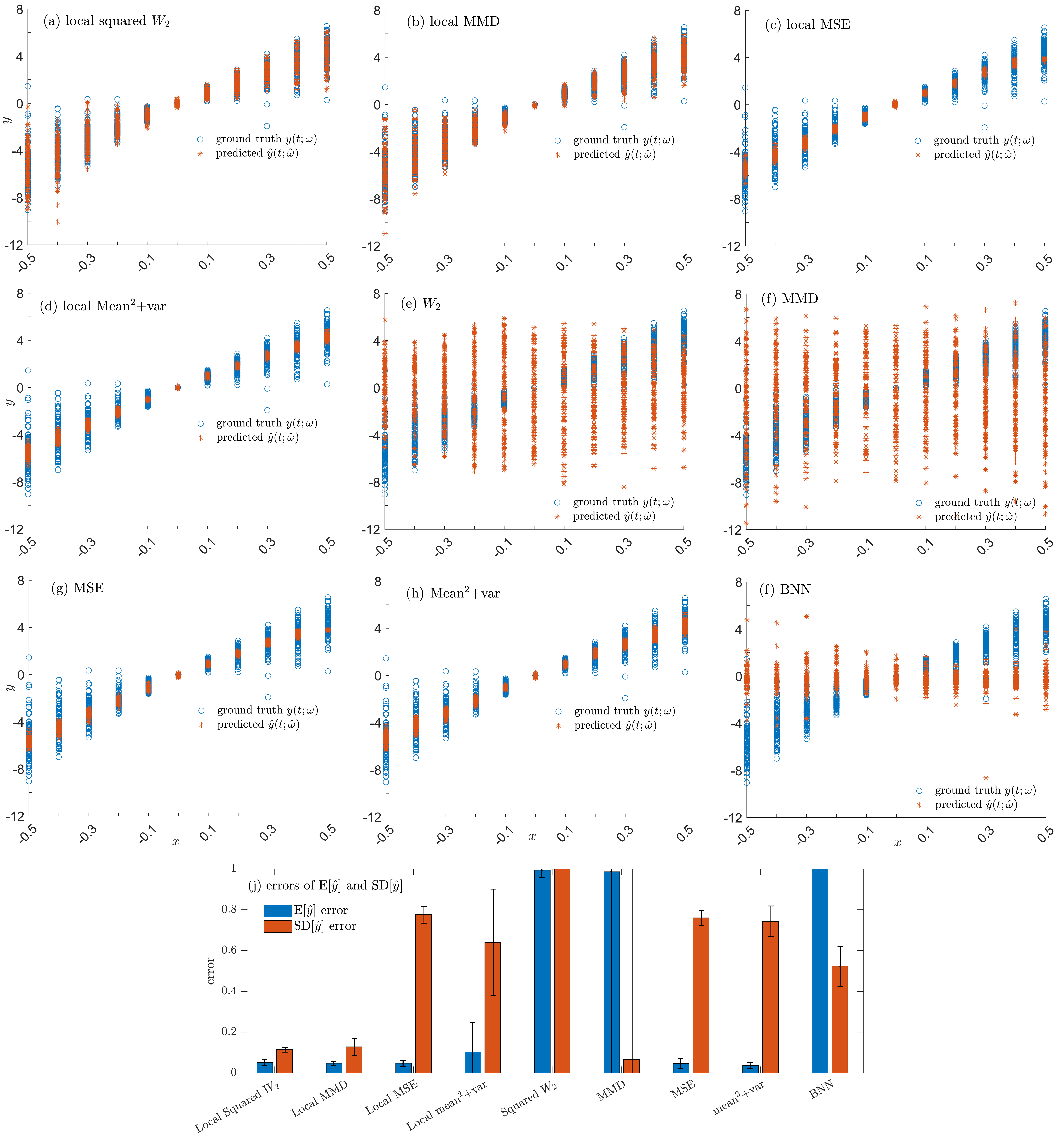}
    \caption{(a)-(i) The ground-truth $y(x;\omega)$ plotted against the predicted $\hat{y}(x;\hat{\omega})$ on the testing set. The predicted $\hat{y}$ is obtained by minimizing different loss functions (defined in Supplement~\ref{def_loss} and obtained by using the BNN method. (j) The average errors in the mean and the standard deviations of $\hat{y}$ on the testing set obtained by minimizing different loss functions and obtained by using the BNN method. The neural network model with weight uncertainty (Fig.~\ref{fig:nn_model}) trained by minimizing our local squared $W_2$ loss yields the smallest errors among all methods. Minimizing the local MMD is comparable to minimizing the local squared $W_2$ loss, likely because the MMD could also somehow measure the discrepancy between two probability distributions. However, unlike the analysis of our local squared $W_2$ method in Subsection~\ref{W2_loss}, there is no theoretical guarantee explaining why the local MMD loss could be successful.}
    \label{fig:example1}
\end{figure}

We compare our local squared $W_2$ loss function with other commonly used loss functions in UQ (definitions given in \ref{def_loss}) as well as a BNN method in \cite{mullachery2018bayesian, BNN_realization} which minimizes the Kullback-Leibler divergence. The neural network model in Fig.~\ref{fig:nn_model} trained by minimizing the local squared $W_2$ loss function yields $\hat{y}(x; \hat{\omega})$ whose distribution is close to the distribution of the ground truth $y(x;\omega)$ on the testing set (Fig.~\ref{fig:example1} (a)).
The performance of minimizing the local MMD loss is comparable to minimizing our local squared $W_2$ loss (Fig.~\ref{fig:example1} (b)).  
The distributions of the predicted $\hat{y}(\bm{x};\hat{\omega})$ by minimizing the local MSE and the local Mean$^2$+var deviate much from the distribution of the ground truth $y(\bm{x};\omega)$ at different $x$ (Fig.~\ref{fig:example1} (c)-(d)). Adopting any ``nonlocal" loss functions yields poor performance (Fig.~\ref{fig:example1} (e)-(h)). The BNN method generates $\hat{y}(\bm{x};\hat{\omega})$ whose distribution fails to match well with the distribution of the ground truth $y(\bm{x}; \omega)$. Overall, our local squared $W_2$ method can most efficiently train the neural network model with weight uncertainty in Fig.~\ref{fig:nn_model} to reconstruct the nonlinear model Eq.~\eqref{example1_model} among all loss functions and methods, with the smallest errors in $\E[\hat{y}(x; \hat{\omega})]$ and $\text{SD}[\hat{y}(x; \hat{\omega})]$ on the testing set (shown in Fig.~\ref{fig:example1} (j)). Additionally, when adopting the neural network model (Fig.~\ref{fig:nn_model}), our method does not require prior knowledge of the form of the nonlinear model Eq.~\eqref{example1_model}, nor does it demand prior distributions of the two latent model parameters $\omega_1, \omega_2$.

%gives the most accurate reconstructed $\hat{y}(x; \hat{\omega})$ whose errors in the mean and standard deviation are the smallest at $x_i=-0.5+0.1i, i=0,...,10$ among all methods. Also, the distributions of $\hat{y}(x;\hat{\omega})$ is close to the distributions of $y(x;\omega)$ 

Two additional experiments are performed. First, we alter the standard deviations of the two parameters $\omega_1, \omega_2$ in Eq.~\eqref{example1_model}
and the standard deviation of the input $x$. We find that larger standard deviations in the latent model parameters and a larger standard deviation in the input $x$ both lead to a poorer reconstruction of the nonlinear model Eq.~\eqref{example1_model}, as shown in Supplement~\ref{model_parameter}. 

Second, we adjust the structure of the neural network model depicted in Fig.~\ref{fig:nn_model}. We discover that using a neural network with 2 hidden layers and 50 neurons per hidden layer equipped with the ResNet technique \cite{he2016deep}  leads to the smallest errors in the reconstructed $\E[\hat{y}(x;\hat{\omega})]$ and $\text{SD}[\hat{y}(x;\hat{\omega})]$. These results are presented in Supplement~\ref{nn_structure}.

\subsection{Application: reconstructing the distributions of concrete compressive strength associated with selected variables}
    \label{example3}
    As an application of our method, we reconstruct the distribution of the concrete compressive strength associated with selected continuous variables in the concrete compressive strength dataset \cite{misc_concrete_compressive_strength_165}. 
    This dataset documents concrete compressive strength along with various influential factors affecting it.
    We reconstruct the distribution of the concrete compressive strength (measured in MPa) based on six recorded continuous variables (measured in kg/m$^3$): cement, fly ash, water, superplasticizer, coarse aggregate, and fine aggregate. Previous models, such as those presented \cite{yeh1998modeling, chang1996mix}, depict concrete compressive strength as a continuous function of these variables. 
    We exclude two discrete, integer-valued variables: blast furnace slag and age. Additionally, other factors that might affect the concrete compressive strength are not recorded in this dataset.
    Thus, the concrete compressive strength might not be a deterministic function of the six selected variables. 
    Instead, we can regard the six selected variables as the input $\bm{x}$, the neglected variables as the latent variables $\omega$, and the concrete compressive strength as $y$ in Eq.~\eqref{model_object}. Then, we may use the approximate model Eq.~\eqref{approximate} to approximate the distribution of the concrete compressive strength given $\bm{x}$. 
 
 We compare the neural network model with weight uncertainty in Fig.~\ref{fig:nn_model}, trained by minimizing the local squared $W_2$ loss function Eq.~\eqref{localw2}, against a neural network without weight uncertainty (\textit{i.e.}, setting $\sigma_{i, j, k}\equiv 0$ for the weights $w_{i, j, k}$ in Fig.~\ref{fig:nn_model}), trained by minimizing the MSE loss (defined in Supplement~\ref{def_loss}). When using a neural network without weight uncertainty, the approximate model is deterministic:
    \begin{equation}
        \hat{y}=\hat{f}(\bm{x}).
    \end{equation}

    The training set $S$ consists of the first two-thirds of samples in the dataset. The remaining one-third of the samples constitute the testing set, denoted by $T$. 
    When calculating the errors in the predicted mean and standard deviation defined in Eq.~\eqref{error_metric}, we use 
$\E\Big[y(\tilde{\bm{x}};\omega)\Big|\big|\tilde{\bm{x}}-\bm{x}|_{x}\leq\delta_0\Big]$
     to approximate $\E[y(\bm{x},\omega)]$, and $
        \E\Big[\hat{y}(\tilde{\bm{x}};\hat{\omega})\Big|\big|\tilde{\bm{x}}-\bm{x}|_{x}\leq\delta_0\Big]$ to approximate $\E[\hat{y}(\bm{x},\omega)]$, respectively. We also use $\text{SD}\Big[y(\tilde{\bm{x}};\omega)\Big|\big|\tilde{\bm{x}}-\bm{x}|_{x}\leq\delta_0\Big]$ to approximate $\text{SD}[y(\bm{x},\omega)]$ and $\text{SD}\Big[\hat{y}(\tilde{\bm{x}};\hat{\omega})\Big|\big|\tilde{\bm{x}}-\bm{x}|_{x}\leq\delta_0\Big]$ to approximate $\text{SD}[\hat{y}(\bm{x},\hat{\omega})]$. Only $(\bm{x},y)\in T$ for which there are at least 5 samples $(\tilde{\bm{x}}, y(\tilde{\bm{x}}, \omega))\in T$ satisfying $|\tilde{\bm{x}}-\bm{x}|_{x}\leq\delta_0$ are used for calculating the errors in the predicted mean and standard deviation. We take $\delta_0=0.2$ and $|\bm{x}|_{x}^2\coloneqq \sum_{i=1}^6c_i^2x_i^2$, where $c_i$ is obtained by minimizing:
    \begin{equation}
\sum_{j=1}^{|S|}\Big(y_j(\bm{x}_j) - \sum_{i=1}^6 c_ix_{j, i} - c_0\Big)^2.
    \end{equation}

    \begin{figure}
    \centering
\includegraphics[width=\linewidth]{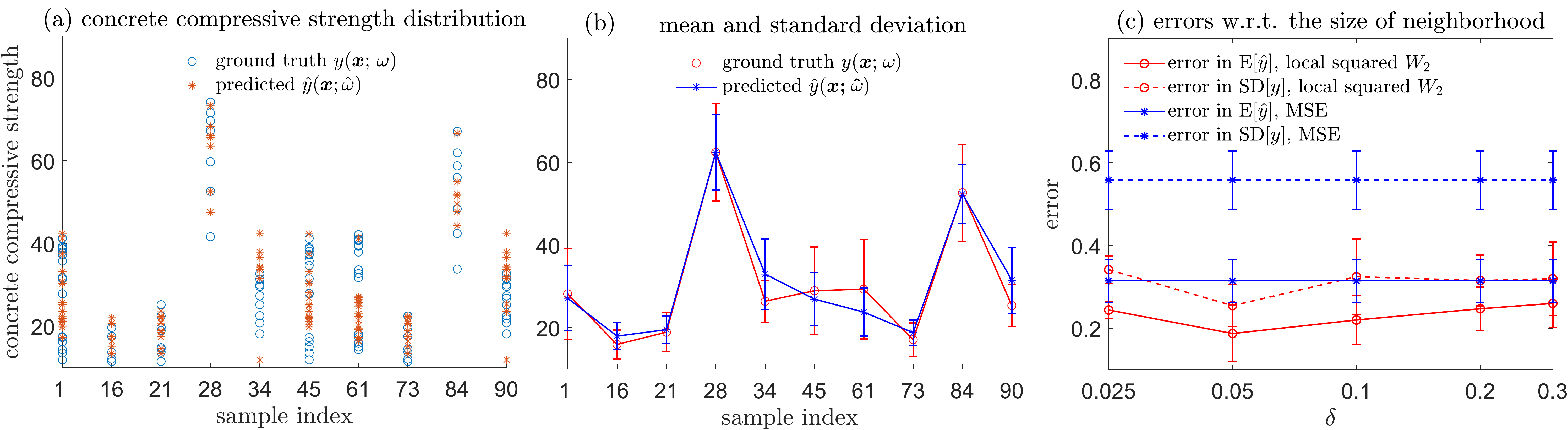}
    \caption{(a) The predicted $\hat{\bm{y}}(\tilde{\bm{x}}; \hat{\omega})$ using the neural network with weight uncertainty in Fig.~\ref{fig:nn_model} versus the ground truth $\bm{y}(\tilde{\bm{x}};\omega)$ for 10 randomly selected samples $(\bm{x}, y)$ in the testing set $T$ satisfying $\Big|\big\{\tilde{\bm{x}}\in T: |\tilde{\bm{x}} - \bm{x}|_{x}\leq\delta_0  \big\} \Big|\geq 5$.  (b) The predicted means $\E\Big[\hat{y}(\tilde{\bm{x}};\hat{\omega})\Big|\big|\tilde{\bm{x}}-\bm{x}|_{x}\leq\delta_0\Big]$ and the predicted standard deviations $\text{SD}[\hat{y}(\tilde{\bm{x}};\hat{\omega})\Big|\big|\tilde{\bm{x}}-\bm{x}|_{x}\leq\delta_0\Big]$ using the neural network with weight uncertainty model in Fig.~\ref{fig:nn_model} versus the ground truth mean $\E[y(\tilde{\bm{x}};\omega)\Big|\big|\tilde{\bm{x}}-\bm{x}|_{x}\leq\delta_0\Big]$ and standard deviation $\text{SD}[y(\tilde{\bm{x}};\omega)\Big|\big|\tilde{\bm{x}}-\bm{x}|_{x}\leq\delta_0\Big]$ for 10 randomly selected samples $(\bm{x}, y)$ in the testing set $T$ satisfying $\Big|\big\{\tilde{\bm{x}}\in T: |\tilde{\bm{x}} - \bm{x}|_{x}\leq\delta_0  \big\} \Big|\geq 5$. In both (a) and (b), we use the neural network with weight uncertainty model trained by minimizing the local squared $W_2$ loss Eq.~\eqref{localw2} with $\delta=0.05$ ($\delta$ is the size of the neighborhood in the loss function and is different from $\delta_0=0.2$). (c) The average relative errors in the mean and standard deviation of predictions $\hat{y}$ generated by the neural network model with weight uncertainty, trained by minimizing the local $W_2$ loss Eq.~\eqref{localw2} versus the average relative errors in the mean and standard deviation of predictions $\hat{y}$ generated by the neural network model without weight uncertainty, trained by minimizing the MSE loss. Note that in (c), the size of neighborhood $\delta$ only applies to using the neural network model with weight uncertainty trained by minimizing the local $W_2$ loss Eq.~\eqref{localw2}. Thus, when using the neural network model without weight uncertainty trained by minimizing the MSE loss, the results do not change with $\delta$.}
    \label{fig:example3}
\end{figure}

Our local squared $W_2$ method yields distributions of $\hat{y}(\tilde{\bm{x}}, \hat{\omega}), |\tilde{\bm{x}}-\bm{x}|_{x}\leq\delta_0$ that align well with the distributions of the ground truth $y(\tilde{\bm{x}}, \omega), |\tilde{\bm{x}}-\bm{x}|_{x}\leq\delta_0$ on the testing set for different $\bm{x}$. As illustrations, in Fig.~\ref{fig:example3} (a)(b), we plot the distributions of the predicted $\hat{y}(\tilde{\bm{x}}, \hat{\omega}), |\tilde{\bm{x}}-\bm{x}|_{x}\leq\delta_0$ against the distributions of the ground truth $y(\tilde{\bm{x}}, \omega), |\tilde{\bm{x}}-\bm{x}|_{x}\leq\delta_0$ for 10 randomly selected samples $(\bm{x}_i, y_i)\in T$ such that the cardinality of the set $\big|\big\{(\bm{x}_j, y_j)\in T: |\bm{x}_j-\bm{x}|_{x}\leq\delta_0\big\}\big|\geq 5$.
 In Fig.~\ref{fig:example3} (c), 
we plot the errors in the predicted mean $\E[\hat{y}(\bm{x}; \hat{\omega})]$ and standard deviation $\text{SD}[\hat{y}(\bm{x}; \hat{\omega})]$ obtained from the two methods: (i)
using the neural network model with weight uncertainty trained by minimizing the local squared $W_2$ loss Eq.~\eqref{localw2} and (ii) using the neural network model without weight uncertainty trained by minimizing the MSE loss. 
The error in the predicted standard deviation 
is much smaller when using (i) than using (ii). 
Thus, our proposed local squared $W_2$ method could more accurately reconstruct the distribution of the concrete compressive strength associated with different $\bm{x}$, compared to using a neural network model without weight uncertainty trained by minimizing the MSE. 

 Similar to the results shown in Fig.~\ref{fig:example2} (b) on reconstructing the linear model Eq.~\eqref{example2_model}, it is most appropriate to choose a moderate $\delta$ in the loss function Eq.~\eqref{localw2} when using our local squared $W_2$ method to reconstruct the distribution of concrete compressive strength on the six selected variables. Errors in the predicted mean and standard deviation can both be well controlled
 when $\delta\in[0.05, 0.1]$, shown in Fig.~\ref{fig:example3} (c). When $\delta$ is too small or too large, the accuracy of the predicted mean and standard deviation decreases.

\subsection{Reconstructing an ODE with parameter uncertainty}
\label{example4}
Finally, we consider an ODE with uncertain latent parameters:
\begin{eqnarray}
    \frac{\text{d}\bm{y}(\bm{y}_0, t; \omega)}{\text{d}t} = \bm{g}(\bm{y}(\bm{y}_0, t; \omega), t, \omega),\,\, \omega\in\Omega, t\in[0, T], \,\,
    \bm{y}(\bm{y}_0, 0; \omega)=\bm{y}_0\in\mathbb{R}^n,
    \label{ode_uncertain}
\end{eqnarray}
where $\omega$ are latent parameters with uncertainty. We aim at using another ODE to approximate Eq.~\eqref{ode_uncertain}:
\begin{eqnarray}
    \frac{\text{d}\hat{\bm{y}}(\bm{y}_0, t; \hat{\omega})}{\text{d}t} = \hat{\bm{g}}\big(\hat{\bm{y}}( \bm{y}_0, t;\hat{\omega}), t, \hat{\omega}\big), \,\,\hat{\omega}\in\hat{\Omega}, t\in[0, T],\,\,
    \hat{\bm{y}}(\bm{y}_0, 0; \hat{\omega})=\bm{y}_0\in\mathbb{R}^n,\,\,
    \label{approximate_ode_model}
\end{eqnarray}
where $\hat{\omega}$ are uncertain parameters in $\hat{\bm{g}}$. In the following, we regard the initial condition $\bm{y}_0$ as the input and $\bm{y}(\bm{y}_0, t;\omega)$ as the output (the norms of the input $\bm{y}_0$ and output $\bm{y}(\bm{y}_0,t;\omega)$ are the same $l^2$ norm $\|\cdot\|$ for vectors). Fixing $t\in[0, T]$, if there exists a Lipschitz constant $L$ such that:
\begin{equation}
    \big\|\bm{y}(\bm{y}_0, t; \omega)-\bm{y}( \tilde{\bm{y}}_0,t; \omega)\big\|\leq L\|\bm{y}_0-\tilde{\bm{y}}_0\|,\,\, \big\|\hat{\bm{y}}(\bm{y}_0, t;\hat{\omega})-\hat{\bm{y}}( \tilde{\bm{y}}_0,t; \hat{\omega})\big\|\leq L\|\bm{y}_0-\tilde{\bm{y}}_0\|,\,\, \forall \bm{y}_0, \tilde{\bm{y}}_0\in\mathbb{R}^n,
\end{equation}
then Theorem~\ref{theorem_1} 
implies that minimizing the local squared $W_2$ distance $\tilde{W}_{2, \delta}^{2, \text{e}}(\bm{y}(\bm{y}_0, t;\omega), \hat{\bm{y}}(\bm{y}_0, t;\hat{\omega}))$ in Eq.~\eqref{localw2} could be effective in comparing the distributions of $\bm{y}(\bm{y}_0, t;\omega)$ and $\hat{\bm{y}}(\bm{y}_0, t;\hat{\omega})$. Additionally, if 
$\bm{g}$ and $\hat{\bm{g}}$ in Eqs.~\eqref{ode_uncertain} and Eq.~\eqref{approximate_ode_model} are uniformly Lipschitz continuous in $\bm{y}$ and $t$, then a large $\tilde{W}_{2, \delta}^{2, \text{e}}(\bm{y}(\bm{y}_0, t;\omega), \bm{y}(\bm{y}_0, t;\hat{\omega}))$ implies that there exists a pair $(\bm{y}, s)\in\mathbb{R}^d\times [0, t]$ such that $\hat{\eta}_{\bm{y}, s}$ fails to align well with $\eta_{\bm{y}, s}$ ($\eta_{\bm{y}, s}$ and $\hat{\eta}_{\bm{y}, s}$ denote the distributions of $\bm{g}(\bm{y}, s, \omega)$ and $\hat{\bm{g}}(\bm{y}, s, \hat{\omega})$, respectively). More analysis on this is provided in Supplement~\ref{ode_proof}.

    We reconstruct the following 4D ODE with a latent random variable (Example 4.3 in \cite{sonday2011eigenvalues}):
    \begin{equation}
        \begin{aligned}
            &\frac{\text{d}{y_1}}{\text{d} t} = (0.05 + \omega)  y_1 + 0.05  y_3  - (1 - \omega^2)  y_2,\\
            &\frac{\text{d}{y_2}}{\text{d} t} = (1 - \omega^2)  y_0 + 0.05  y_4,\\
            &\frac{\text{d}{y_3}}{\text{d} t} = (-0.05 + \omega)  y_3 - (1-\omega^2)  y_4,\\
            &\frac{\text{d}{y_4}}{\text{d} t} = (1-\omega^2)  y_3, \,\,\,\, t\in[0, 2].
        \end{aligned}
        \label{4dodemodel}
    \end{equation}
    Let $\bm{g}(\bm{y}, \omega)\coloneqq\big(g_1(\bm{y}, \omega), g_2(\bm{y}, \omega), g_3(\bm{y}, \omega), g_4(\bm{y}, \omega)\big)^T$ represent the RHS of Eq.~\eqref{4dodemodel}. We set the initial condition $\bm{y}_0\sim\mathcal{N}((1,1,1,1)^T, a^2 I_4)$, where $I_4\in\mathbb{R}^{4\times4}$ denotes the identity matrix.
    In Eq.~\eqref{4dodemodel}, we let $\omega\sim\mathcal{U}(-\sigma, \sigma)$. We independently sample the initial condition $\bm{y}_0$ and $\omega$, generating 100 trajectories for both the training and testing sets.
     The neural network model with weight uncertainty in Fig.~\ref{fig:nn_model} is adopted as the RHS $\hat{\bm{g}}$ in the approximate ODE model Eq.~\eqref{approximate_ode_model}, which aims at approximating Eq.~\eqref{4dodemodel} (we also set $\hat{\bm{g}}$ to be time-homogeneous, \textit{i.e.}, $\hat{\bm{g}} = \hat{\bm{g}}(\bm{y},\hat{\omega})$).
The means and variances of the weights as well as the biases in the neural network are optimized by minimizing the time-averaged local squared $W_2$ distance:
\begin{equation}
    \frac{1}{m+1}\sum_{t=0}^m \tilde{W}_{2, \delta}^{2, \text{e}}\big(\bm{y}(\bm{y}_0, t_i;\omega), \hat{\bm{y}}(\bm{y}_0, t_i; \hat{\omega})\big),\,\,t_i=i\Delta t,\,\, \Delta t=\frac{2}{m}.
    \label{time_w2}
\end{equation}
    The following error metrics:
    \begin{equation}
        \text{error in}~\hat{\bm{y}}\coloneqq  \frac{\int_0^2\tilde{W}_{2,\delta_0}^{2, \text{e}}\big(\bm{y}(\bm{y}_0, s;\omega), \hat{\bm{y}}(\bm{y}_0, s;\hat{\omega})\big)\text{d}s}{\int_0^2\tilde{W}_{2,\delta_0}^{2, \text{e}}\big(\bm{y}(\bm{y}_0, s;\omega), 0\big)\text{d}s},\,\,\,\,\,\,
    \text{error in}~\hat{\bm{g}}\coloneqq \frac{\int_0^2\E\big[W_2^2(\eta_{\bm{y}(s), s}, \hat{\eta}_{\bm{y}(s), s})\big]\text{d}s}{\int_0^2\E\big[\|\bm{g}(\bm{y}(s), s,\omega)\|^2\big]\text{d}s}
    \label{ode_error}
\end{equation}
are used to quantify the errors of $\hat{\bm{y}}$ and $\hat{\bm{g}}$ in the reconstructed ODE \eqref{approximate_ode_model}, respectively. We set $\delta=\delta_0=0.1$ in Eqs.~\eqref{time_w2} and~\eqref{ode_error} and $m=100$ in Eq.~\eqref{time_w2}.

    \begin{figure}[h!]
    \centering
\includegraphics[width=\linewidth]{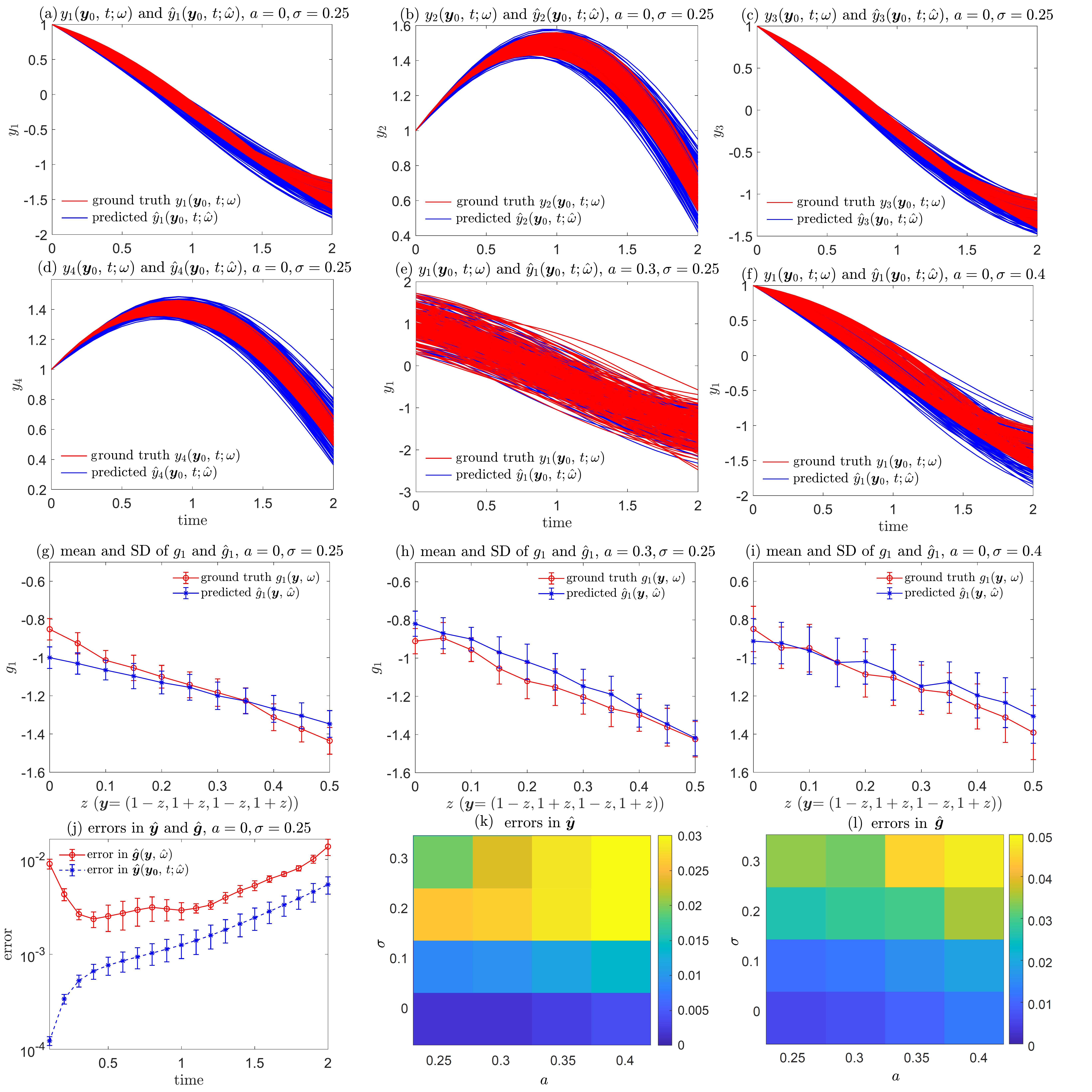}
    \caption{(a)-(d) Comparison between the ground truth $y_i(\bm{y}_0, t;\omega)$ and the predicted $\hat{y}_i(\bm{y}_0,t;\hat{\omega}), i=1,2,3,4$ when the standard deviation of the initial condition $a=0$ and $\sigma=0.25$ in the distribution of the model parameter $\omega$.  (e) Comparison between the ground truth $y_1(\bm{y}_0, t;\omega)$ and the predicted $\hat{y}_1(\bm{y}_0,t;\hat{\omega})$ ($a=0.3, \sigma=0.25$). (f) The ground truth $y_1(\bm{y}_0, t;\omega)$ versus the predicted $\hat{y}_1(\bm{y}_0,t;\hat{\omega})$ ($a=0, \sigma=0.4$).
    In (a)-(f), the ground truth $y_i(\bm{y}_0, t;\omega)$ are trajectories in the testing set and the predicted $\hat{y}_i(\bm{y}_0, t;\omega)$ are 
    generated based on the initial conditions in the testing set (the testing set and the training set share the same $a, \sigma$).
    (g) Means and standard deviations of the ground truth $g_1(\bm{y}, \omega)$ versus the predicted $\hat{g}_1(\bm{y}, \hat{\omega})$ ($a=0, \sigma=0.25$). (h) Means and standard deviations of the ground truth $g_1(\bm{y}, \hat{\omega})$ versus the predicted $\hat{g}_1(\bm{y}, \hat{\omega})$ ($a=0, \sigma=0.4$). (i) Means and standard deviations of the ground truth $g_1(\bm{y}, \omega)$ versus the predicted $\hat{g}_1(\bm{y}, \hat{\omega})$ ($a=0.3,\sigma=0.25$). In (g)-(i), we let $\bm{y}=(1-z, 1+z, 1-z, 1+z), z=0.05i, i=0,\ldots,10$. (j) The errors $\frac{\tilde{W}_2^2\big(\bm{y}(\bm{y}_0, t;\omega), \hat{\bm{y}}(\bm{y}_0, t; \hat{\omega})\big)}{\E\big[\|\bm{y}(\bm{y}_0, t;\omega)\|^2\big]}$ and $\frac{\E\Big[W_2^2\big(\eta_{\bm{y}(t), t}, \hat{\eta}_{\bm{y}(t), t}\big)\Big]}{\E\big[\|\bm{g}(\bm{y}(t), t,\omega)\|^2\big]}$ in $\hat{\bm{y}}$ and $\hat{\bm{g}}$ at different time $t$ when $a=0, \sigma=0.25$. $\eta_{\bm{y}, s}$ and $\hat{\eta}_{\bm{y}, s}$ denote the distributions of $\bm{g}(\bm{y}, s, \omega)$ and $\hat{\bm{g}}(\bm{y}, s, \hat{\omega})$, respectively. (k) Errors in $\hat{\bm{y}}$ (defined in Eq.~\eqref{ode_error}) for different $a$ and $\sigma$ . (l) Errors in $\hat{\bm{g}}$ and $\hat{\bm{y}}$ (defined in Eq.~\eqref{ode_error}) for different $a$ and $\sigma$. The errors of $\hat{\bm{y}}$ and $\hat{\bm{g}}$ are evaluated on the testing sets.}
    \label{fig:example4}
\end{figure}

Overall, by minimizing the time-averaged local squared $W_2$ distance Eq.~\eqref{time_w2} and using the neural network with weight uncertainty as $\hat{\bm{g}}$, 
the distribution of trajectories generated by our reconstructed model, Eq.~\eqref{approximate_ode_model}, closely aligns with the distribution of trajectories generated by the ground truth ODE~\eqref{4dodemodel}, across different values of $a, \sigma$. To demonstrate, we plot the ground truth $y_i(\bm{y}_0, t; \omega)$ and the reconstructed $\hat{y}_i(\bm{y}_0, t;\hat{\omega}), i=1,2,3,4$ when: $a=0, \sigma=0.25$ (Fig.~\ref{fig:example4}(a)(b)(c)(d)). We also plot the ground truth $y_1(\bm{y}_0, t; \omega)$ against the reconstructed $\hat{y}_1(\bm{y}_0, t;\hat{\omega})$ for $a=0, \sigma=0.4$ (Fig.~\ref{fig:example4}(e)) and $a=0.3, \sigma=0.25$ (Fig.~\ref{fig:example4}(f)). 
Additionally, the distributions of the ground truth $\bm{g}$
are effectively represented by the distribution of the reconstructed
$\hat{\bm{g}}$ when inputting the same $\bm{y}$ for different values of $\sigma$ and $a$. As an example, we plot the means and standard deviations of ground truth $g_1$ against those of the predicted $\hat{g}_1$ along the line $\bm{y}=(1-z, 1+z, 1-z, 1+z)$ when: $a=0, \sigma=0.25$ (Fig.~\ref{fig:example4} (g)), $a=0.3, \sigma=0.25$ (Fig.~\ref{fig:example4} (h)), and $a=0, \sigma=0.4$ (Fig.~\ref{fig:example4} (i)). 
 In Fig.~\ref{fig:example4} (j), the error in $\hat{\bm{y}}$ grows over time due to error accumulation but is kept below 0.1 for all $t$.
 From Fig.~\ref{fig:example4} (k)
(l), larger values of $a$ and $\sigma$ correspond to larger
 errors in $\hat{\bm{y}}$ and $\hat{\bm{g}}$. 
One potential explanation is that larger values of 
$a$ and 
$\sigma$
result in sparser training trajectories, rendering it more challenging to accurately reconstruct the underlying model  Eq.~\eqref{4dodemodel}.

\section{Discussion}
\label{conclusion}
In our paper, we proposed a local squared $W_2$ method for reconstructing uncertainty models in UQ through minimizing a local squared $W_2$ loss Eq.~\eqref{localw2}. The local squared $W_2$ loss function
could be efficiently evaluated using empirical distributions of observed data. We showcased the effectiveness of our approach across various UQ tasks and showed that it outperformed some benchmark methods.

 As future directions, it would be promising to conduct further analysis to determine the optimal size of neighborhood $\delta$ in Eq.~\eqref{localw2} as well as to identify an appropriate norm $|\cdot|_x$ for the input $\bm{x}$ for our method. It would be beneficial to develop an appropriate surrogate model as $\hat{\bm{f}}$ in Eq.~\eqref{approximate}.
Furthermore, exploring the application of our local squared $W_2$ method to other UQ problems merits further investigation. For example, integrating our local squared $W_2$ method with recent stochastic differential equation reconstruction methods \cite{xia2024squared, xia2024an} could enable the reconstruction of dynamical systems characterized by both uncertain parameters and intrinsic fluctuations.

\section{Methods}
\subsection{Local squared $W_2$ method for uncertainty quantification}
\label{W2_loss}

In this subsection, we analyze the novel local squared $W_2$ method we propose in Subsection~\ref{local_definition} for reconstructing the uncertainty model Eq.~\eqref{model_object}. First, we formally define the $W$-distance between two $d$-dimensional random variables.
\begin{definition} 
\rm 
\label{def:W2}
For two random variables $\bm{y}, \hat{\bm{y}}\in\mathbb{R}^n$,
we assume that
\begin{equation}
    \E[\|\bm{y}\|^2]\leq \infty,\,\,\,\,\E[\|\hat{\bm{y}}\|^2]\leq \infty,
\end{equation}
where the norm $\|\cdot\|$ is the $l^2$ norm of a vector.
We denote probability distributions of $\bm{y}$ and $\hat{\bm{y}}$ by $\mu$ and  $\hat{\mu}$, respectively. 
The $W_2$-distance
$W_2(\mu, \hat{\mu})$ is defined as
\begin{equation}
W_p(\mu, \hat{\mu}) \coloneqq \inf_{\pi(\mu, \hat \mu)}
\E_{(\bm{y}, \hat{\bm
{y}})\sim \pi(\mu, \hat \mu)(\bm{y}, \hat{\bm
{y}})}\big[\|{\bm{y}} - \hat{{\bm{y}}}\|^{2}\big]^{\frac{1}{2}}.
\label{pidef}
\end{equation}
$\pi(\mu, \hat \mu)(\bm{y}, \hat{\bm
{y}})$ iterates over all \textit{coupled} distributions
of $\bm{y}, \hat{\bm{y}}$, defined by the condition
\begin{equation}
\begin{cases}
{\pi(\mu, \hat \mu)}\left(A \times  \mathbb{R}^n\right) ={\mu}(A),\\
{\pi(\mu, \hat \mu)}\left( \mathbb{R}^n\times A\right) = {\hat \mu}(A), 
\end{cases}\forall A\in \mathcal{B}( \mathbb{R}^n),
\label{pi_def}
\end{equation}
where $\mathcal{B}(\mathbb{R}^n)$ denotes the Borel
$\sigma$-algebra associated with $\mathbb{R}^n$.
\end{definition}

To simplify our analysis, we make the following assumptions.
\begin{assumption}
\rm
\label{assumptions_w2}
We assume that the following conditions hold for the uncertainty model Eq.~\eqref{model_object} and the approximate model Eq.~\eqref{approximate}.

\begin{enumerate}
    \item We assume that $\bm{y}$ and $\hat{\bm{y}}$ in Eqs.~\eqref{model_object} and \eqref{approximate} are uniformly bounded:
\begin{equation}
    \|\bm{y}\|\leq \sqrt{M}, \,\,\,\, \|\hat{\bm{y}}\|\leq \sqrt{M},
    \label{upperboundy}
\end{equation}
where $\|\cdot\|$ denotes the $l^2$ norm of a vector.
\item We assume that $\bm{f}$ and $\hat{\bm{f}}$ on the RHSs of Eqs.~\eqref{model_object} and \eqref{approximate} are uniformly Lipschitz continuous on $\bm{x}$, \textit{i.e.}, there exists $L\leq\infty$ such that
\begin{equation}
    \|\bm{f}(\bm{x};\omega) - \bm{f}(\hat{\bm{x}};\omega)\|\leq L |\bm{x}-\hat{\bm{x}}|_{x}, \,\, \|\hat{\bm{f}}(\bm{x};\omega) - \hat{\bm{f}}(\hat{\bm{x}};\omega)\|\leq L |\bm{x}-\hat{\bm{x}}|_x,\,\, \forall x, \hat{x}\in D, 
    \label{l_condition}
\end{equation}
where $|\cdot|_{x}$ is a norm for $\bm{x}\in\mathbb{R}^d$.
\item The random variable $\omega$ is independent of $\bm{x}$; the random variable $\hat{\omega}$ is also independent of $\bm{x}$.
\end{enumerate}
\end{assumption}

We denote the distributions of $\bm{y}(\bm{x}; \omega)$ and $\hat{\bm{y}}(\bm{x}; \hat{\omega})$ by $\mu_{\bm{x}}$ and $\hat{\mu}_{\bm{x}}$, respectively.
From Eq.~\eqref{pidef}, $W_2^2(\mu_{\bm{x}}, \hat{\mu}_{\bm{x}})\geq 0$ and $W_2^2(\mu_{\bm{x}}, \hat{\mu}_{\bm{x}})= 0$ only when $\bm{y}=\hat{\bm{y}}$ almost surely under a coupling measure $\pi(\mu_{\bm{x}}, \hat{\mu}_{\bm{x}})(\bm{y}, \hat{\bm{y}})$. This indicates that $\mu_{\bm{x}}=\hat{\mu}_{\bm{x}}\,\, \text{a.s.}$ since the marginal distributions of $\pi(\mu_{\bm{x}}, \hat{\mu}_{\bm{x}})(\bm{y}, \hat{\bm{y}})$ are $\mu_{\bm{x}}$ and $\hat{\mu}_{\bm{x}}$, respectively. Generally, the smaller $W_2^2(\mu_{\bm{x}}, \hat{\mu}_{\bm{x}})$ is, the more similar the probability measure $\hat{\mu}_{\bm{x}}$ is to the probability measure $\mu_{\bm{x}}$. 

Next, we consider the following quantity (the same as Eq.~\eqref{approx_quantity}):
\begin{equation}
    \tilde{W}_2^2(\bm{y}, \hat{\bm{y}}) \coloneqq \int_D W_2^2\big(\mu_{\bm{x}}, \hat{\mu}_{\bm{x}}\big)\nu(\text{d}\bm{x}),
    \label{squarew2loss}
\end{equation}
where $\nu$ is the distribution of $\bm{x}\in D$. We assume that $\nu$ is a non-degrading probability measure on $D$. Then,  $\tilde{W}_2^2(\bm{y}, \hat{\bm{y}})$ defined in Eq.~\eqref{squarew2loss} equals to 0 only when 
\begin{equation}
    W_2^2\big(\mu_{\bm{x}}, \hat{\mu}_{\bm{x}}\big)=0,\,\, \text{a.e.}, x\in D.
\end{equation}
Thus, $\tilde{W}_2^2(\bm{y}, \hat{\bm{y}})$ is minimized only when then the distribution of $\bm{y}(\bm{x};\omega)$ in the uncertainty model Eq.~\eqref{model_object} can be perfectly matched by the distribution of $\hat{\bm{y}}(\bm{x};\omega)$ in the approximate model Eq.~\eqref{approximate} a.e. in $D$.

However, it is usually difficult to evaluate the loss function Eq.~\eqref{squarew2loss} when only a finite set of observations $S\coloneqq\{(\bm{x}_i, \bm{y}_i)\}_{i=1}^N$ is available. If $\bm{x}\in D$ is a continuous random variable with a non-degrading probability density function $\nu(\bm{x})$, then almost surely $\bm{x}_i\neq \bm{x}_j$ for any $\bm{x}_i, \bm{x}_j\in S$ when $i\neq j$. Consequently, it is challenging to evaluate $\mu_{\bm{x}}$.
To tackle this problem, we propose the local squared $W_2$ distance loss function Eq.~\eqref{localw2} as an approximation to $\tilde{W}_2^2(\bm{y}, \hat{\bm{y}})$ in Eq.~\eqref{squarew2loss} (also Eq.~\eqref{approx_quantity}). We can prove the following theorem that gives an error bound of using the local squared $W_2$ loss function $\tilde{W}_{2, \delta}^{2, \text{e}}(\bm{y}, \hat{\bm{y}})$ in Eq.~\eqref{localw2} to approximate $\tilde{W}_2^2(\bm{y}, \hat{\bm{y}})$.

\begin{theorem}
\rm
    \label{theorem_1}
    For each $x\in D$, we denote the number of samples $(\tilde{\bm{x}}, \tilde{\bm{y}})\in S$ such that $|\tilde{\bm{x}}-\bm{x}|_x\leq\delta$ to be $N(\bm{x}, \delta)$. We denote the total number of samples of the empirical distribution to be $N$. Assuming that each input $\bm{x}$ is independently sampled from the probability distribution $\nu$, then
    we have the following error bound
    \begin{equation}
    \E\Big[\big|\tilde{W}_2^2(\bm{y}, \hat{\bm{y}}) - \tilde{W}_{2, \delta}^{2, \text{e}}(\bm{y}, \hat{\bm{y}})\big|\Big] \leq \frac{4M}{\sqrt{N}} + 8CM\E\big[h(N(x, \delta), n)\big] + 8\sqrt{M}L\delta
        \label{theorem2_result}
    \end{equation}
    where $\tilde{W}_{2, \delta}^{2, \text{e}}(\bm{y}, \hat{\bm{y}})$ is the local $W_2$ distance defined in Eq.~\eqref{localw2} and $\tilde{W}_{2}^{2}(\bm{y}, \hat{\bm{y}})$ is defined in Eq.~\eqref{approx_quantity}.
    $M$ is the upper bound for $\|\bm{y}\|, \|\hat{\bm{y}}\|$ in Eq.~\eqref{upperboundy}, $C$ is a constant, $N$ is the total number of training data $(\bm{x}, \bm{y})$, and $L$ is the Lipschitz constant in Eq.~\eqref{l_condition}. In Eq.~\eqref{theorem2_result}, 
    \begin{equation}
h(N, d)\coloneqq\left\{
\begin{aligned}
&2N^{-\frac{1}{4}}\log(1+N)^{\frac{1}{2}}, d\leq4,\\
&2N^{-\frac{1}{d}}, d> 4
\end{aligned}
\right.
\label{t_def}
\end{equation}
\end{theorem}

The proof to Theorem~\ref{theorem_1} is provided in Supplement~\ref{proof_theorem1}. 
Specifically, there is a trade-off between the second and third terms in the error bound Eq.~\eqref{theorem2_result}: if we increase $\delta$, then $N(\bm{x}, \delta)$ tends to increase, which makes the second term smaller but the third term larger. Nonetheless, Theorem~\ref{theorem_1}
implies that $\tilde{W}_2^2(\bm{y}, \hat{\bm{y}})$ can be well approximated by $\tilde{W}_{2, \delta}^{2, \text{e}}(\bm{y}, \hat{\bm{y}})$ when the number of training data $N$ is sufficiently large such that even for a small $\delta$, $\E[N(\bm{x}, \delta)^{-\frac{1}{2n}}]$ can be maintained small. In this scenario, minimizing our local squared $W_2$ loss function is also necessary such that the distribution of $\bm{y}(\bm{x};\omega)$ can be well represented by the distribution of $\hat{\bm{y}}(\bm{x};\hat{\omega})$ for different $\bm{x}$.

\subsection{Structure of the neural-network model with weight uncertainty}
 The structure of the neural network model with weight uncertainty used in Section~\ref{numerical_result} is given below in Fig.~\ref{fig:nn_model}. When we use this neural-network model with weight uncertainty as the approximate model $\hat{f}$ in Eq.~\eqref{approximate},
 all weights with uncertainty $\{w_{i, j, k}\}$ constitutes the random variable $\hat{\omega}$ in Eq.~\eqref{approximate}. The mean and variance $a_{i, j, k}, \sigma_{i, j, k}^2$ of the weight $w_{i, j, k}$ as well as the bias $b_{i, k}$ for all $i, j, k$ are to be optimized through minimizing the local squared $W_2$ loss function Eq.~\eqref{localw2} (or other loss functions in Subsection~\ref{example1}).
    \begin{figure}[h!]
    \centering
\includegraphics[width=0.9\linewidth]{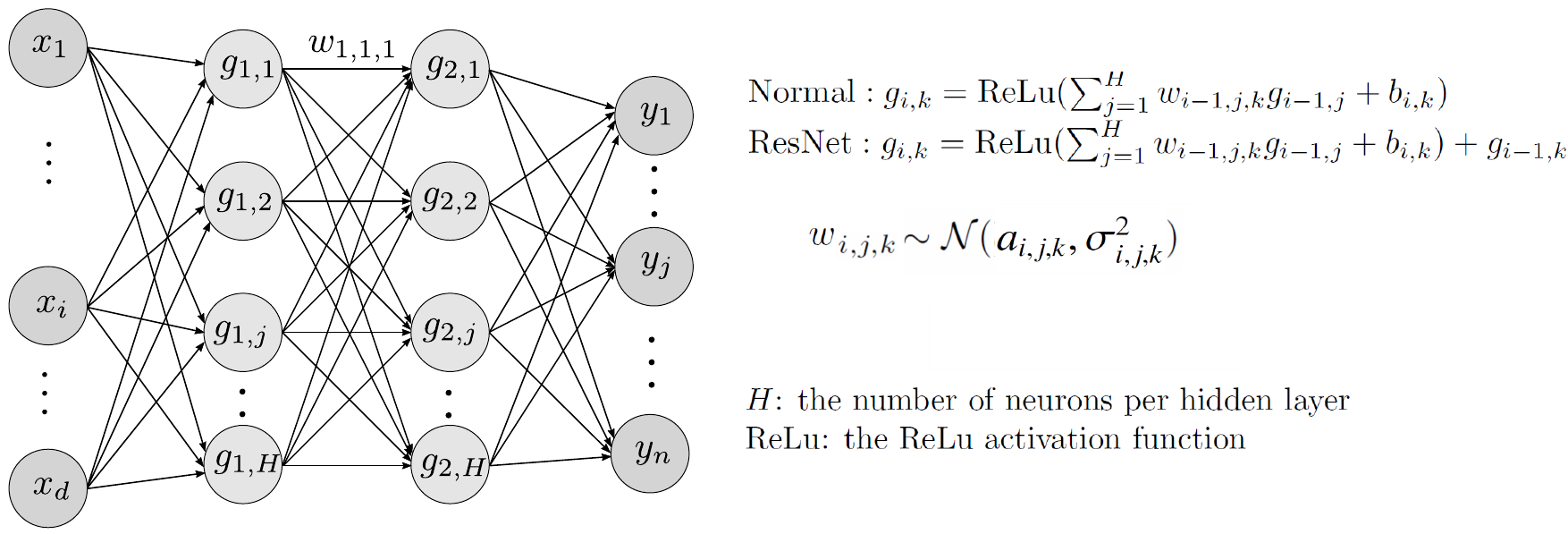}
    \caption{A sketch of the structure of the neural network model with weight uncertainty used in this paper. The weights $w_{i, j, k}\sim\mathcal{N}(a_{i, j, k}, \sigma_{i, j, k}^2)$ are independently sampled, \textit{i.e.}, $w_{i_1, j_1, k_1}$ is independent of $w_{i_2, j_2, k_2}$ when $(i_1, j_1, k_1)\neq (i_2, j_2, k_2)$. When using this neural network model to make predictions, for each input $\bm{x}=(x_1,\ldots,x_d)\in D\subseteq\mathbb{R}^d$, we shall resample all weights $\{w_{i, j, k}\}$ again.
    Either the normal feed-forward structure for forward propagation or the ResNet technique \cite{he2016deep} for forward propagation is adopted.}
    \label{fig:nn_model}
\end{figure}

\section*{Data Availability}
No new data were created in this research. The concrete compressive strength data used in Subsection~\ref{example3} are publicly available in \cite{misc_concrete_compressive_strength_165}.

\section*{Code Availability}
The code used in this research will be made publicly available upon acceptance of this manuscript.

\section*{Acknowledgement}
The authors thank Prof. Tom Chou from UCLA for his valuable suggestions on this work.

\bibliography{bibliography}
\appendix
\newpage
\renewcommand\thesection{S\arabic{section}}
\renewcommand\thesubsection{S\arabic{section}.\arabic{subsection}}
\renewcommand\thesubsubsection{S\arabic{section}.\arabic{subsection}.\arabic{subsubsection}}

\setcounter{equation}{0}
\renewcommand{\theequation}{S\arabic{equation}}

\section{Proof to Theorem~\ref{theorem_1}}
\label{proof_theorem1}
Here, we shall prove Theorem~\ref{theorem_1}. First, we have 
\begin{equation}
    \E\Big[\big|\tilde{W}_2^2(\bm{y}, \hat{\bm{y}}) - \tilde{W}_{2, \delta}^{2, \text{e}}(\bm{y}, \hat{\bm{y}})\big|\Big] \leq \E\Big[\big|\tilde{W}_2^2(\bm{y}, \hat{\bm{y}}) - \tilde{W}_{2}^{2, \text{e}}(\bm{y}, \hat{\bm{y}})\big|\Big] + \E\Big[\big|\tilde{W}_2^{2, \text{e}}(\bm{y}, \hat{\bm{y}}) - \tilde{W}_{2, \delta}^{2, \text{e}}(\bm{y}, \hat{\bm{y}})\big|\Big],
    \label{estimate_term}
\end{equation}
where 
\begin{equation}
    \tilde{W}_{2}^{2, \text{e}}(\bm{y}, \hat{\bm{y}})\coloneqq \int_{D}W_2^2\big(\mu_{\bm{x}}, \hat{\mu}_{\bm{x}}\big)\nu^{\text{e}}(\text{d}\bm{x}),
\end{equation}
 and $\nu^{\text{e}}(\text{d}x)$ is the empirical distribution of $\bm{x}$.
For the first term in Eq.~\eqref{estimate_term}, the following inequality holds:
\begin{equation}
\begin{aligned}
\E\Big[\big|\int_{D}W_2^2\big(\mu_{\bm{x}}, \hat{\mu}_{\bm{x}}\big)\nu^{\text{e}}(\text{d}\bm{x}) - \int_{D}W_2^2\big(\mu_{\bm{x}}, \hat{\mu}_{\bm{x}}\big)\nu(\text{d}\bm{x})\big|\Big]
    &\leq \E\Big[\big(\int_{D}W_2^2\big(\mu_{\bm{x}}, \hat{\mu}_{\bm{x}}\big)\nu^{\text{e}}(\text{d}\bm{x}) - \int_{D}W_2^2\big(\mu_{\bm{x}},\hat{\mu}_{\bm{x}}\big)\nu(\text{d}\bm{x})\big)^2\Big]^{\frac{1}{2}}\\
    &\quad\leq \frac{1}{\sqrt{N}}\E\bigg[\Big(W_2^2\big(\mu_{\bm{x}}, \hat{\mu}_{\bm{x}}\big) - \E[W_2^2\big(\mu_{\bm{x}}, \hat{\mu}_{\bm{x}}\big)]\Big)^2\bigg]^{\frac{1}{2}}
    \leq \frac{4M}{\sqrt{N}}.
    \label{intermediate1}
\end{aligned}
\end{equation}
The last inequality holds because for any $\bm{x}\in D$, using the assumption Eq.~\eqref{upperboundy}, we have
\begin{equation}
   0\leq W_2^2\big(\mu_{\bm{x}}, \hat{\mu}_{\bm{x}}\big) \leq 2 \Big(\E\big[\|y(\bm{x};\omega)\|^2\big] + \E\big[\|\hat{y}(\bm{x};\hat{\omega})\|^2\big]\Big)=4M.
\end{equation}

Next, we estimate the second term in Eq.~\eqref{estimate_term}:
\begin{equation}
    \E\bigg[\Big|\int_{D}W_2^2\big(\mu_{\bm{x}}, \hat{\mu}_{\bm{x}})\big)\nu^{\text{e}}(\text{d}x) - \int_{D}W_2^2\big(\mu_{\bm{x}, \delta}^{\text{e}}, \hat{\mu}_{\bm{x}, \delta}^{\text{e}}\big)\nu^{\text{e}}(\text{d}x)\Big|\bigg].
\end{equation}
We denote $\mu_{\bm{x}, \delta}$ ($\mu_{\bm{x}, \delta}^{\text{e}}$) to be the conditional distribution (empirical conditional distribution) of $\bm{y}(\tilde{\bm{x}}; \omega)$ conditioned on $|\tilde{\bm{x}}-\bm{x}|\leq \delta$ and $\hat{\mu}_{\bm{x}, \delta}$ ($\hat{\mu}_{\bm{x}, \delta}^{\text{e}}$ ) to be the conditional distribution (empirical conditional distribution) of $\hat{\bm{y}}(\tilde{\bm{x}}; \hat{\omega})$ conditioned on $|\tilde{\bm{x}}-\bm{x}|\leq \delta$, respectively.

For any $\bm{x}\in D$, we have
\begin{equation}
\begin{aligned}
&\hspace{-2cm}\Big|W_2^2\big(\mu_{\bm{x}}, \hat{\mu}_{\bm{x}}\big) - W_2^2\big(\mu_{\bm{x}, \delta}^{\text{e}}, \hat{\mu}_{\bm{x}, \delta}^{\text{e}}\big)\Big|\leq |W_2\big(\mu_{\bm{x}}, \hat{\mu}_{\bm{x}}\big) - W_2\big(\mu_{\bm{x}, \delta}^{\text{e}}, \hat{\mu}_{\bm{x}, \delta}^{\text{e}}\big)|\cdot\bigg(W_2\big(\mu_{\bm{x}}, \hat{\mu}_{\bm{x}}\big) + W_2\big(\mu_{\bm{x}, \delta}^{\text{e}}, \hat{\mu}_{\bm{x}, \delta}^{\text{e}}\big)\bigg)
\\&\hspace{2.2cm}\leq4\sqrt{M}\Big|W_2\big(\mu_{\bm{x}}, \hat{\mu}_{\bm{x}}\big) - W_2\big(\mu_{\bm{x}, \delta}^{\text{e}}, \hat{\mu}_{\bm{x}, \delta}^{\text{e}}\big)\Big|.
\end{aligned}
    \label{break_ineq}
\end{equation}

Using the triangle inequality of the Wasserstein distance \cite{clement2008elementary}, for any $\bm{x}$, we have
\begin{equation}
    \begin{aligned}
        &\hspace{-1cm}\big|W_2(\mu_{\bm{x}}, \hat{\mu}_{\bm{x}}) - W_2(\mu_{\bm{x}, \delta}^{\text{e}}, \hat{\mu}_{\bm{x},  \delta}^{\text{e}}) \big|\leq \big|W_2(\hat{\mu}_{\bm{x}}, \mu_{\bm{x}}) - W_2(\hat{\mu}_{\bm{x}}, \mu_{\bm{x}, \delta})\big| + \big|W_2(\hat{\mu}_{\bm{x}}, \mu_{\bm{x}, \delta}) - W_2(\mu_{\bm{x}, \delta}, \hat{\mu}_{\bm{x}, \delta})\big| \\
    &\hspace{4cm}+ \big|W_2(\mu_{\bm{x}, \delta}, \hat{\mu}_{\bm{x}, \delta}) - W_2(\hat{\mu}_{\bm{x}, \delta}, \mu_{\bm{x}, \delta}^{\text{e}})\big|  + \big|W_2(\hat{\mu}_{\bm{x}, \delta}, \mu_{\bm{x}, \delta}^{\text{e}}) - W_2(\mu_{\bm{x}, \delta}^{\text{e}}, \hat{\mu}^{\text{e}}_{\bm{x}, \delta})\big|\\
    &\hspace{3cm}\leq W_2(\mu_{\bm{x}, \delta}, {\mu}_{\bm{x}}) +  W_2(\hat{\mu}_{\bm{x}, \delta}, \hat{\mu}_{\bm{x}}) + W_2(\mu_{\bm{x},\delta}^{\text{e}}, \mu_{\bm{x}, \delta}) + W_2(\hat{\mu}^{\text{e}}_{\bm{x}, \delta}, \hat{\mu}_{\bm{x}, \delta})
    \label{misalign}
    \end{aligned}
\end{equation}

We shall estimate the first term in the last inequality of Eq.~\eqref{misalign}. We define a new random variable coupled with $\bm{y}(\bm{x}, \omega)$ such that given $\bm{x}, \omega$ of $\bm{y}(\bm{x}, \omega)$:
\begin{equation}
    \tilde{\bm{y}}(\tilde{\bm{x}}; \tilde{\omega}) \coloneqq \bm{f}(\tilde{\bm{x}}, \omega),
\end{equation}
where the random variable $\tilde{\bm{x}}$ is independent of $\bm{x}$ and independent of $\omega$, and we let $\tilde{\bm{x}}$ have a probability density $\frac{\mathbb{I}_{|\tilde{\bm{x}}-\bm{x}|\leq \delta}}{P(A_{\bm{x}})}\cdot \nu(\text{d}\tilde{\bm{x}})$. $A_x$ denotes the set $\{\tilde{\bm{x}}\in D:|\tilde{\bm{x}}-\bm{x}|_x\leq \delta\}$ and 
$\mathbb{I}_{|\tilde{\bm{x}}-\bm{x}|_x\leq \delta}$ is the indicator function:
\begin{equation}
    \mathbb{I}_{|\tilde{\bm{x}}-\bm{x}|_x\leq \delta}=1, |\tilde{\bm{x}}-\bm{x}|_{x}\leq \delta,\,\,\,\,\mathbb{I}_{|\tilde{\bm{x}}-\bm{x}|_x\leq \delta}=0, |\tilde{\bm{x}}-\bm{x}|_{x}>\delta. 
\end{equation}
Since $\bm{f}$ is Lipschitz in $\bm{x}$, we have
\begin{equation}
    \big\|\tilde{\bm{y}}(\tilde{\bm{x}}; \tilde{\omega}) - \bm{y}(\bm{x}; \omega)\big\|\leq L\delta.
\end{equation}
Additionally, the distribution of $\tilde{\bm{y}}$ is also $\mu_{\bm{x}, \delta}$ because $\tilde{x}$ and $\omega$ are independent.

 We take a special coupling probability measure $\tilde{\pi}(\mu_{\bm{x}}, \mu_{\bm{x}, \delta}) = (\bm{y}, \bm{\tilde{y}})_*P$ such that for all $A\in \mathcal{B}(\mathbb{R}^n\times\mathbb{R}^n)$, 
\begin{equation}
    \tilde{\pi}(\mu_{\bm{x}}, \mu_{\bm{x}, \delta})(A)=P\big((\bm{y}, \tilde{\bm{y}})^{-1}(A)\big),
\end{equation}
where $(\bm{y}, \tilde{\bm{y}})$ is interpreted as a measurable map from $\Omega\times(\mathbb{R}^d\times\Omega)$ to $\mathbb{R}^n\times\mathbb{R}^n$. 
$(\bm{y}, \bm{\tilde{y}})^{-1}(A)\in \Omega\times(\mathbb{R}^d\times\Omega)$ is the preimage of $A$ under $(\bm{y}, \bm{\tilde{y}})$ and $P\big((\bm{y}, \tilde{\bm{y}})^{-1}(A)\big)$ is the probability measure of the set $(\bm{y}, \tilde{\bm{y}})^{-1}(A)$. Therefore, we have
\begin{equation}
    W_2(\mu_{\bm{x}, \delta}, \mu_{\bm{x}})\leq \E_{(\bm{y}, \tilde{\bm{y}})\sim \tilde{\pi}(\mu_{\bm{x}}, \mu_{\bm{x}, \delta})}\Big[\big\|\bm{y}-\tilde{\bm{y}}\big\|^2\Big]^{\frac{1}{2}}\leq L\delta.
    \label{mis_error_bound}
\end{equation}
Similarly, for the second term in the last inequality of Eq.~\eqref{misalign}, from the Lipschtiz continuity condition of $\hat{\bm{f}}$ in Assumption Eq.~\eqref{l_condition}, we can also show that 
\begin{eqnarray}
    W_2(\hat{\mu}_{\bm{x}, \delta}, \hat{\mu}_{\bm{x}})\leq L\delta.
    \label{mis_error_bound2}
\end{eqnarray}

For the third and fourth terms in the last inequality of Eq.~\eqref{misalign}, from Theorem 1 in \cite{fournier2015rate}, there exists a constant $C$ such that 
\begin{eqnarray}
    \E\Big[W_2(\mu_{\bm{x},\delta}^{\text{e}}, \mu_{\bm{x},\delta})\Big]\leq \E\Big[W_2^2(\mu_{\bm{x},\delta}^{\text{e}}, \mu_{\bm{x},\delta})\Big]^{\frac{1}{2}}\leq C\E\Big[\|\bm{y}(\bm{x};\omega)\|_6^6\Big]^{\frac{1}{6}}h(N(x, \delta),n)\leq C\sqrt{M}h\big(N(x, \delta),n\big)
    \label{sample_bound}
\end{eqnarray}
and 
\begin{eqnarray}
\E\Big[W_2(\hat{\mu}_{\bm{x},\delta}^{\text{e}}, \hat{\mu}_{\bm{x},\delta})\Big]\leq \E\Big[W_2^2(\hat{\mu}_{\bm{x},\delta}^{\text{e}}, \hat{\mu}_{\bm{x},\delta})\Big]^{\frac{1}{2}}\leq C\E\Big[\|\hat{\bm{y}}(\bm{x};\hat{\omega})\|_6^6\Big]^{\frac{1}{6}}h(N(x, \delta), n)\leq C\sqrt{M}h\big(N(x, \delta),n\big),
    \label{sample_bound2}
\end{eqnarray}
respectively. Here, $\|\cdot\|_6$ is the $l^6$ norm of a vector and we have $\|\bm{y}\|_6\leq\|\bm{y}\|$. In Eq.~\eqref{sample_bound2}, the function $h$ is defined as:
\begin{equation}
h(N, n)=\left\{
\begin{aligned}
&N^{-\frac{1}{4}}\log(1+N)^{\frac{1}{2}}, n\leq4,\\
&N^{-\frac{1}{n}}, n> 4
\end{aligned}
\right.
\end{equation}

Plugging Eqs.~\eqref{sample_bound}, \eqref{sample_bound2}, \eqref{mis_error_bound}, and \eqref{mis_error_bound2} into Eq.~\eqref{misalign}, we have proved that:
\begin{equation}
    \E\Big[\big|W_2(\mu_{\bm{x}}, \hat{\mu}_{\bm{x}}) - W_2(\mu_{\bm{x}, \delta}^{\text{e}}, \hat{\mu}_{\bm{x}, \delta}^{\text{e}}) \big|\Big]\leq 2C\sqrt{M}h(N(x, \delta), n) + 2L\delta.
    \label{error_inter}
\end{equation}
Therefore, combining the two inequalities Eqs.~\eqref{break_ineq},\eqref{error_inter} and taking the expectation of Eq.~\eqref{error_inter} w.r.t. the empirical distribution $\nu^{\text{e}}(\bm{x})$, we have
\begin{eqnarray}
    \E\Big[\big|\tilde{W}_{2}^{2, \text{e}}(\bm{y}, \hat{\bm{y}}) - \tilde{W}_{2, \delta}^{2, \text{e}}(\bm{y}, \hat{\bm{y}})\big|\Big]\leq 8CM\E\big[h(N(x, \delta), n)\big] + 8\sqrt{M}L\delta.
    \label{intermediate2}
\end{eqnarray}
Combining the two inequalities Eqs.~\eqref{intermediate1} and \eqref{intermediate2}, the inequality~\eqref{theorem2_result} holds, thus completing the proof of Theorem~\ref{theorem_1}.

\section{Definitions of different loss metrics}
\label{def_loss}
Here, we provide descriptions and definitions for different loss functions used in this study. In the following, $N$ denotes the number of samples. 
% enumerate with tight spacing
\begin{compactenum}
\item The squared $W_2$ distance
$$W_2^2(\mu, \hat{\mu})\approx W_2^2(\mu^{\text{e}}, \hat{\mu}^{\text{e}}),$$ where
$\mu^{\text{e}}$ and $\hat{\mu}^{\text{e}}$ are the empirical
distributions of $\bm{y}$ and
$\hat{\bm{y}}$, respectively. It is
estimated by
\begin{equation}
W_2^2(\mu_N^{\text{e}},
\hat{\mu}_N^{\text{e}})\approx\texttt{ot.emd2}\Big(\frac{1}{N}\bm{I}_{N},
\frac{1}{N}\bm{I}_{N}, \bm{C}\Big),
\label{time_coupling}
\end{equation}
where $\texttt{ot.emd2}$ is the function for solving the earth movers
distance problem in the $\texttt{ot}$ package of Python \cite{flamary2021pot}. $N$ is the
number of ground truth and predicted samples, $\bm{I}_{N}$ is
an $N$-dimensional vector whose elements are all 1, and
$\bm{C}\in\mathbb{R}^{N\times N}$ is a matrix with entries
$(\bm{C})_{ij} = \|\bm{y}_i-\hat{\bm{y}}_j\|^2$. $\bm{y}_i$ and $\hat{\bm{y}}_j$ are the ground truth data and prediction associated with $\bm{x}_j$, respectively.
\item MMD (maximum mean discrepancy) 
  \cite{li2015generative}:
$$
\text{MMD}(\bm{y}, \hat{\bm{y}}) = \E[K(\bm{y}, \bm{y})]
- 2\E[K(\bm{y}, \hat{\bm{y}})] + \E[K(\hat{\bm{y}}, \hat{\bm{y}})],
$$
where $K$ is the standard radial basis function (or Gaussian kernel)
with the multiplier $2$ and number of kernels $5$. $\bm{y}$ and $\hat{y}$
are the ground truth observation and prediction, respectively.
\item Mean squared error (MSE): where $N$ is
  the total number of data:
$$\operatorname{MSE}(\bm{y}, \hat{\bm{y}}) = \sum_{i=1}^N \frac{1}{N}
\|y_{i}-\hat y_{i}\|^2$$.
\item Mean$^2$+var loss function:
\begin{eqnarray*}(\operatorname{mean}^2+\operatorname{var})(\bm{y}, \hat { \bm{y}}) =
  \sum_{i=1}^N \sum_{i=1}^N \frac{1}{N}
\|y_{i}-\hat y_{i}\|^2 + |\text{Var}(\bm{y}) - \text{Var}(\hat{\bm{y}})|
    \end{eqnarray*}
where
\begin{eqnarray}
    \text{Var}(\bm{y})  = \sum_{i=1}^N \big\|\bm{y}_i - \sum_{i=1}^N \frac{1}{N}\bm{y}_i\big\|^2,\,\,\,\,
    \text{Var}(\hat{\bm{y}})  = \sum_{i=1}^N \big\|\hat{\bm{y}}_i - \sum_{i=1}^N \frac{1}{N}\hat{\bm{y}}_i\big\|^2
\end{eqnarray}
\item The local squared $W_2$ distance
$$\tilde{W}_{2, \delta}^{2, \text{e}}(\bm{y}, \hat{y})$$ defined in Eq.~\eqref{localw2}. 
It is
estimated by
\begin{equation}
\tilde{W}_{2,\delta}^{2,{\text{e}}}(\bm{y},
\hat{\bm{y}})\approx\sum_{i=1}^N\frac{1}{N}\texttt{ot.emd2}\Big(\frac{1}{N(\bm{x}_i, \delta)}\bm{I}_{N(\bm{x}_i, \delta)},
\frac{1}{N(\bm{x}_i, \delta)}\bm{I}_{N(\bm{x}_i, \delta)}, \bm{C}(\bm{x}_i)\Big),
\label{time_coupling}
\end{equation}
where $\texttt{ot.emd2}$ is the function for solving the earth movers
distance problem. $N$ is the
number of ground truth and predicted samples, $\bm{I}_{N(\bm{x}_i, \delta)}$ is
an $N(\bm{x}_i, \delta)$-dimensional vector whose elements are all 1, and
$\bm{C}(\bm{x}_i)\in\mathbb{R}^{N(\bm{x}_i, \delta)\times N(\bm{x}_i, \delta)}$.
 For each $\bm{x}_i$, we denote $S(\bm{x}_i, \delta)\coloneqq\{\bm{x}_{i_1},\ldots,\bm{x}_{i_{N(\bm{x}_i, \delta)}}\}$ to be the set such that $ S(\bm{x}_i, \delta)=\{\tilde{x}\in S: |\tilde{\bm{x}}-\bm{x}_i|_x\leq\delta\}$. The entries in $\bm{C}$ are: $(\bm{C})_{sj} = \|\bm{y}_{i_s}-\hat{\bm{y}}_{i_j}\|^2$. $\|\cdot\|$ is the $l^2$ norm of a vector. $|\cdot|_x$ is the norm of the input $\bm{x}$ (specified in each example).
$\bm{y}_{i_s}$ and $\hat{\bm{y}}_{i_s}$ are the ground truth $\bm{y}_{i_s}=\bm{f}(\bm{x}_{i_s}, \omega_{i_s})$ and predicted $\hat{\bm{y}}_{i_s}=\hat{\bm{f}}(\bm{x}_{i_s}, \hat{\omega}_{i_s})$ for $\bm{x}_{i_s}\in S(\bm{x}_i, \delta), s=1,\ldots,N(\bm{x}_i, \delta)$.
\item Local MMD:
$$
\text{MMD}_{\delta}(\bm{y}, \hat{\bm{y}}) = \sum_{i=1}^N \frac{1}{N} \bigg(\E\Big[K\big(\bm{y}[\bm{x}_i, \delta], \bm{y}[\bm{x}_i, \delta]\big)\Big]
- 2\E\Big[K\big(\bm{y}[\bm{x}_i, \delta], \hat{\bm{y}}[\bm{x}_i, \delta]\big)\Big] + \E\Big[K\big(\hat{\bm{y}}[\bm{x}_i, \delta], \hat{\bm{y}}[\bm{x}_i, \delta]\big)\Big]\bigg),
$$
where $K$ is the standard radial basis function (or Gaussian kernel)
with multiplier $2$ and number of kernels $5$. $\bm{y}[\bm{x}_i, \delta]$ is the set of ground truth $\big\{\bm{y}_{i_s}=\bm{f}(\bm{x}_{i_s};\omega_{i_s})\big\}$ such that $\bm{x}_{i_s}\in S(\bm{x}_i, \delta)$. $\hat{\bm{y}}[\bm{x}_{i_s}, \delta]$ is the set of reconstructed $\big\{\hat{\bm{y}}_{i_s}=\hat{\bm{f}}(\bm{x}_{i_s};\hat{\omega}_{i_s})\big\}$ such that $\bm{x}_{i_s}\in S(\bm{x}_i, \delta)$. $S(\bm{x}_i, \delta)\coloneqq\{\bm{x}_{i_1},\ldots,\bm{x}_{i_{N(\bm{x}_i, \delta)}}\}$ has the same meaning as defined in the local squared $W_2$ distance.  
\item Local MSE:
$$
\text{MSE}_{\delta}(\bm{y}, \hat{\bm{y}}) = \sum_{i=1}^N \frac{1}{N} \text{MSE}\Big(\bm{y}[\bm{x}_i, \delta], \bm{y}[\bm{x}_i, \delta)]\Big)
$$
 $\bm{y}[\bm{x}_i, \delta]$ and  $\hat{\bm{y}}[\bm{x}_i, \delta]$ have the same meaning as defined in the local MMD loss function.
 \item Local Mean$^2$+var:
$$
(\text{Mean}^2+\text{var})_{\delta}(\bm{y}, \hat{\bm{y}}) = \sum_{i=1}^N \frac{1}{N} (\text{Mean}^2+\text{var})\Big(\bm{y}[\bm{x}_i, \delta], \bm{y}[\bm{x}_i, \delta)]\Big)
$$
 $\bm{y}[\bm{x}_i, \delta]$ and  $\hat{\bm{y}}[\bm{x}_i, \delta]$ have the same meaning as defined in the local MMD loss function.
\end{compactenum}

\section{Optimization \& training settings and hyperparameters}
\label{training_details}
In Table~\ref{tab:setting}, we list the training hyperparameters and settings for each example.  All experiments are conducted using Python
3.11 on a desktop equipped with a 32-core Intel® i9-13900KF CPU (when comparing
runtime, we train each model on just one core).
\begin{table}[h!]
\centering
\caption{Training hyperparameters, hyperparameters in the neural network model, and training settings for each example. The neural network parameters include means and standard deviations $a_{i, j, k}, \sigma_{i, j, k}$ for weights $w_{i, j, k}$ as well as biases $b_{i, k}$ in Fig.~\ref{fig:nn_model}.} 
\begin{tabular}{lllll}
\toprule
 & Subsection~\ref{example2} & Subsection~\ref{example1}& Subsection~\ref{example3}& Subsection~\ref{example4} \\
\midrule
gradient descent method & AdamW & AdamW & AdamW  & AdamW\\
forward propagation method & $\backslash$ & ResNet & ResNet & Normal\\
learning rate & 0.02 & 0.025 & 0.02  & 0.005\\
weight decay & 0.005 & 0.005 & 0.005 & 0.005  \\
number of epochs &1000 & 1000 & 1000 & 500  \\
number of training samples & 1000 & 2000 & 686 & 100 \\
size of neighborhood $\delta$ in the loss function Eq.~\eqref{localw2} & 0.1& 0.025& 0.05&0.1\\
number of hidden layers in $\Theta$ & $\backslash$ & 4 & 4 &2 \\
activation function &  $\backslash$&ReLu &ReLu &ReLu \\
number of neurons in each layer in $\Theta_1$ & $\backslash$ & 50 & 50 &100  \\
initialization for model/neural-network parameters & $1$  & $\mathcal{N}(0, 10^{-4})$  & $\mathcal{N}(0, 10^{-4})$ &$\mathcal{N}(0, 10^{-4})$\\
repeat times & 5 & 5 & 5 & 5\\
\bottomrule
\end{tabular}
\label{tab:setting}
\end{table}

\section{How distributions of model parameters and the input affects the accuracy of reconstructing Eq.~\eqref{example1_model}}
\label{model_parameter}
We carry out two additional experiments on reconstructing Eq.~\eqref{example1_model} in Subsection~\ref{example1}, aiming to investigate how the distributions of the model parameters $(\omega_1, \omega_2)$ and the input $x$ affect the accuracy of reconstructing the nonlinear model Eq.~\eqref{example1_model}. 

First, we varied the distribution of the uncertain parameters $(\omega_1, \omega_2)$. We set $x \sim \mathcal{U}(-0.5, 0.5)$ and $(\omega_1, \omega_2)^T \sim \mathcal{N}((19.1426, 0.5311)^T, \sigma^2 I_2)$ in Eq.~\eqref{example1_model}, where $I_2 \in \mathbb{R}^{2\times2}$ is the identity matrix, to generate the training samples. For the testing samples, we set $x\in\{x_i: x_i=-0.5+0.1i, i = 0, \ldots, 10\}$ and $(\omega_1, \omega_2)^T \sim \mathcal{N}\big((19.1426, 0.5311)^T, \sigma^2 I_2\big)$. At each $x_i$, 100 testing samples are generated.
The variables $x$ and $(\omega_1, \omega_2)$ were independently sampled for both the training and testing sets. We varied $\sigma$, the standard deviation of the latent model parameters $\omega_1$ and $\omega_2$, for both the training and testing samples. 

Second, we alter the distribution of $x$ in the training set. We let $(\omega_1, \omega_2)\sim\mathcal{N}\big((19.1426, 0.5311)^T, I_2\big)$ and sample
$x\sim\mathcal{U}(-a, a)$ with different $a$ for the training set. For testing, we generate a testing set $T=\cup_{i=0}^{11} T_i$ with each $T_i$ containing 100 samples $(x_{r, i}, y(x_{r, i}, \omega)), x_{r, i}=0.1i-0.5, r=1,\ldots,100$ and $(\omega_1, \omega_2)\sim\mathcal{N}\big((19.1426, 0.5311)^T, I_2\big)$. $x$ and $(\omega_1, \omega_2)$ are independently sampled for both training and testing sets.

For both experiments, we use the same neural network model with weight uncertainty, training settings, and hyperparameters as used in Subsection~\ref{example1} (listed in Table~\ref{tab:setting}). The errors in the predicted mean $\E[\hat{y}(x; \hat{\omega})]$ and standard deviation $\text{SD}[\hat{y}(x; \hat{\omega})]$ are calculated on the testing set.

\begin{figure}[h!]
    \centering
    \includegraphics[width=0.8\linewidth]{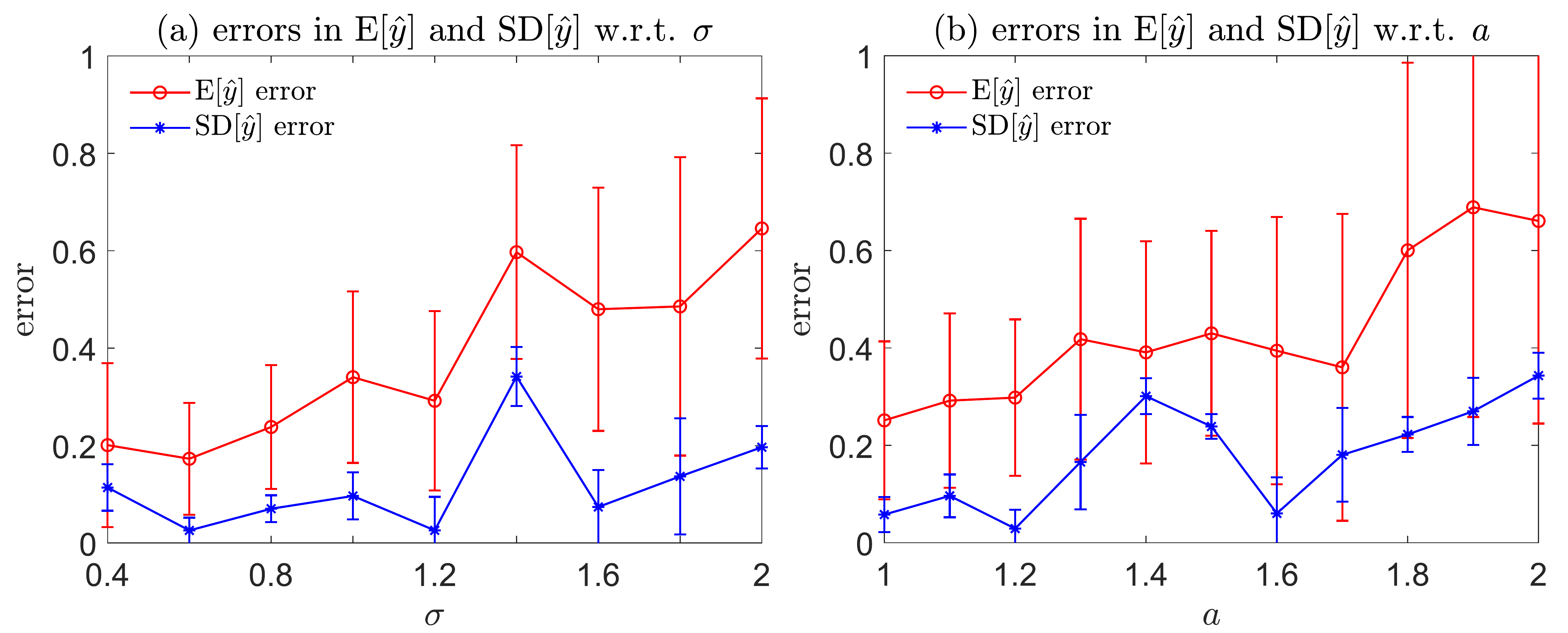}
    \caption{(a) The errors in the predicted mean $\E[\hat{y}(x; \hat{\omega})]$ and predicted standard deviation $\text{SD}[\hat{y}(x; \hat{\omega})]$ w.r.t. the standard deviations of the uncertain model parameters $(\omega_1, \omega_2)$. (b) The errors in the predicted mean $\E[\hat{y}(x; \hat{\omega})]$ and predicted standard deviation $\text{SD}[\hat{y}(x; \hat{\omega})]$ when varying the standard deviation of $x$ in the training set. The errors are evaluated on the testing set.}
    \label{fig:example1_appendixic}
\end{figure}

As depicted in Fig.~\ref{fig:example1_appendixic} (a)(b), larger standard deviations in the model parameters $(\omega_1, \omega_2)$ and a larger standard deviation in the input $x$ of the training set both result in larger errors in the predicted mean $\E[\hat{y}(x; \hat{\omega})]$ and in the predicted standard deviation $\text{SD}[\hat{y}(x; \hat{\omega})]$. A possible reason could be that larger standard deviations in the model parameters $(\omega_1, \omega_2)$ and a larger standard deviation in the input $x$ both lead to more sparsely distributed training samples, making it
more difficult to reconstruct the underlying nonlinear model Eq.~\eqref{example1_model}.

\section{Neural network architecture}
\label{nn_structure}
We carry out an additional experiment to explore how the structure of the neural network model in Fig.~\ref{fig:nn_model} affects the accuracy of reconstructing the nonlinear model Eq.~\eqref{example1_model}. The local squared $W_2$ distance Eq.~\eqref{localw2} is used as the loss function for training the neural network with the size of neighborhood $\delta=0.1$. We shall change the number of hidden layers as well as the number of neurons per hidden layer. 
Additionally, we will compare the performance of the ResNet technique against the standard feed-forward structure for forward propagation.

\begin{table}[h!]
\centering
  \caption{Means and standard deviations of errors in the predicted $\E[\hat{y}(x;\hat{\omega})]$ and $\text{SD}[\hat{y}(x;\hat{\omega})]$ in Subsection~\ref{example1}, calculated over 5 independent experiments. The errors are calculated on the same testing set as used in Subsection~\ref{example1}.}
\begin{tabular}{lrcllr}
\toprule  Width & Depth & Error in $\E[\hat{y}(x;\hat{\omega})]$ & Error in $\text{SD}[\hat{y}(x;\hat{\omega})]$ & runtime (s) \\ 
\midrule 
  12  & 4(ResNet) & \(0.053(\pm 0.013) \) & \( 0.127 (\pm 0.023) \) & \(3099\pm561\) \\ 
  25  & 4(ResNet) & \(0.035(\pm 0.005) \) & \( 0.126 (\pm 0.026) \) & \(2469\pm502\) \\ 
   50 & 4(ResNet) & \( 0.051\pm0.013 \) & \( 0.114\pm0.012 \) & $3346\pm213$\\ 
100  & 4(ResNet) & \( 0.232 (\pm 0.409) \) & \( 0.190 (\pm 0.138) \) & $3706\pm781$ \\ 
50 & 1(ResNet) & \( 0.044\pm0.005 \) & \( 0.106\pm0.018 \) & $2853\pm593$\\ 
 50 & 2(ResNet) & \( 0.034\pm0.003 \) & \( 0.106\pm0.034 \)
& $2801\pm299$ \\ 
50 & 3(ResNet) & \( 0.043\pm0.010 \) & \( 0.115\pm0.024 \) & $2428\pm499$ \\ 
 50 & 1(feed-forward)  &\(0.049\pm0.005\) & \(0.106\pm0.023\)& \(2503\pm 90\)\\
 50 & 2(feed-forward)  &\(0.044\pm0.011\) & \(0.129\pm0.019\)& \(2526\pm 560\)\\
  50 & 3(feed-forward)  &\(0.075\pm0.031\) & \(0.161\pm0.020\)& \(2526\pm 560\)\\
   50 & 4(feed-forward) &\(0.060\pm0.028\)&\(0.175\pm0.022\) & \(2820\pm622\)\\
\bottomrule
\end{tabular}
\label{tab:example_nnstructure}
\end{table}

In Table~\ref{tab:example_nnstructure}, the errors of the predicted mean and standard deviation, $\E[\hat{y}(x;\hat{\omega})]$ and $\text{SD}[\hat{y}](x;\hat{\omega})$, may increase as the number of hidden layers in the neural network increases if the ResNet technique is not implemented. However, when the ResNet technique is employed, the errors of the predicted mean and standard deviation do not deteriorate as the number of hidden layers increases.
 This improvement may be attributed to the ResNet technique mitigating the gradient vanishing issue \cite{borawar2023resnet} that affects simple feed-forward neural networks.
 On the other hand, if the number of neurons is too small or too large, then the errors in $\E[\hat{y}(x;\hat{\omega})]$ and $\text{SD}[\hat{y}(x;\hat{\omega})]$ becomes large. A too-small number of neurons per layer could be insufficient, while an excessively large number of neurons can complicate the optimization of weights under uncertainty. A neural network with 3 hidden layers and 50 neurons in each layer, equipped with the ResNet technique, appears to be the most effective configuration for reconstructing the nonlinear model Eq.~\eqref{example1_model}.

\section{Analysis on the RHSs of the two ODEs Eqs.~\eqref{ode_uncertain} and \eqref{approximate_ode_model}}
\label{ode_proof}
Assume that $\bm{g}$ and $\hat{\bm{g}}$ on the RHSs of Eqs.~\eqref{ode_uncertain} and \eqref{approximate_ode_model} are uniformly Lipschtiz continuous in the first two arguments:
\begin{equation}
\begin{aligned}
    &\big\|\bm{g}(\bm{y}_1,t_1,\omega_t) - \bm{g}(\bm{y}_2,t_2,\omega)\big\|\leq L_g\big(\|\bm{y}_1 - \bm{y}_2\| + |t_2-t_1|\big),\\
    &\hspace{1cm}\big\|\hat{\bm{g}}(\bm{y}_1,t_1,\omega_t)-\hat{\bm{g}}(\bm{y}_2,t_2,\omega)\big\|\leq L_g\big(\|\bm{y}_1 - \bm{y}_2\| + |t_2-t_1|\big),\,\, 0<L_g\leq\infty, \forall \bm{y}_1, \bm{y}_2\in\mathbb{R}^n, \forall t_1, t_2\geq 0.
\end{aligned}
    \label{g_l_condition}
\end{equation}
We can prove the following result.

\begin{proposition}
\rm
\label{Propode}
Assuming that Eq.~\eqref{g_l_condition} holds, there exist two constants $C_0, C_1$ such that:
\begin{equation}
    W_2(\mu_{\bm{y}_0, t}, \hat{\mu}_{\bm{y}_0, t})\leq C_0\sup_{\bm{y}, 0\leq s\leq t}W_2^2(\eta_{\bm{y}, s}, \hat{\eta}_{\bm{y}, s})+C_1,\,\, \forall \bm{y}_0\in\mathbb{R}^n,
    \label{upper_bound_result}
\end{equation}
 In Eq.~\eqref{upper_bound_result}, $\mu_{\bm{y}_0, t}$ and $\hat{\mu}_{\bm{y}_0, t}$ are the distributions of $\bm{y}(\bm{y}_0,t;\omega)$ and $\hat{\bm{y}}(\bm{y}_0,t;\hat{\omega})$; $\eta_{\bm{y}, s},\hat{\eta}_{\bm{y}, s}$ are the distributions of $\bm{g}(\bm{y}, s, \omega)$ and $\hat{\bm{g}}(\bm{y}, s, \omega)$ in Eqs.~\eqref{ode_uncertain} and \eqref{approximate_ode_model}.
\end{proposition}

\begin{proof}
 First, note that
\begin{equation}
    \begin{aligned}
       \text{d}\|\bm{y}(\bm{y}_0, t;\omega)-\hat{\bm{y}}(\bm{y}_0, t;\hat{\omega})\|^2&=2\Big(\bm{y}(\bm{y}_0,t;\omega)-\hat{\bm{y}}(\bm{y}_0,t;\hat{\omega}), \bm{g}(\bm{y}(\bm{y}_0, t;\omega), t, \omega) - \hat{\bm{g}}(\hat{\bm{y}}(\bm{y}_0, t;\hat{\omega}), t, \hat{\omega}) \Big)\\
        &\leq  \big\|\bm{y}(\bm{y}_0,t;\omega)-\hat{\bm{y}}(\bm{y}_0,t;\hat{\omega})\|^2 + \big\|\bm{g}(\bm{y}(\bm{y}_0,t;\omega), t, \omega) \\
        &\hspace{2cm} -\hat{\bm{g}}(\bm{y}(\bm{y}_0,t;\hat{\omega}), t, \hat{\omega}) + \hat{\bm{g}}(\bm{y}(\bm{y}_0,t;\hat{\omega}), t, \hat{\omega})     
        - \hat{\bm{g}}(\hat{\bm{y}}(\bm{y}_0,t;\hat{\omega}), t, \hat{\omega})\big\|^2\\
        &\leq (1 + 2L_g^2)\|\bm{y}(\bm{y}_0,t;\omega)-\hat{\bm{y}}(\bm{y}_0,t;\hat{\omega})\|^2 + 2\|\bm{g}(\bm{y}(\bm{y}_0,t;\omega), t, \omega) - \hat{\bm{g}}(\bm{y}(\bm{y}_0,t;\omega), t, \hat{\omega})\|^2\\
        &\leq (1 + 2L_g^2)\|\bm{y}(\bm{y}_0, t;\omega)-\hat{\bm{y}}(\bm{y}_0,t;\hat{\omega})\|^2 + 2\big\|\bm{g}(\bm{y}(\bm{y}_0,t;\omega), t, \omega) - \bm{g}(0, 0, \omega) \\&\hspace{2cm}+ \bm{g}(0, 0, \omega) -  \hat{\bm{g}}(0, 0, \hat{\omega}) + \hat{\bm{g}}(0, 0, \hat{\omega}) - \hat{\bm{g}}(\bm{y}(\bm{y}_0, t;\omega), t, \hat{\omega})\big\|^2\\
        &\hspace{-2cm}\leq (1 + 2L_g^2)\|\bm{y}(\bm{y}_0, t;\omega)-\hat{\bm{y}}(\bm{y}_0, t;\hat{\omega})\|^2 + 6\|\bm{g}(0, 0, \omega) - \hat{\bm{g}}(0, 0, \hat{\omega})\|^2 + 12L_g^2\big(\|\bm{y}(\bm{y}_0, s;\omega)\|^2+t^2\big) ,
    \end{aligned}
    \label{diff_form}
\end{equation}
where $(\cdot, \cdot)$ denotes the inner product of two $n$-dimensional vectors. By applying the Gronwall's inequality to the quantity $\|\bm{y}(\bm{y}_0,t;\omega)-\hat{\bm{y}}(\bm{y}_0,t;\hat{\omega})\|^2$ in Eq.~\eqref{diff_form}, we can conclude that:
\begin{equation}
    \|\bm{y}(\bm{y}_0,t;\omega)-\hat{\bm{y}}(\bm{y}_0,t;\hat{\omega})\|^2\leq 6\exp(t + 2L_g^2t)\cdot\int_0^t \Big(\|\bm{g}(0, 0, \omega) - \hat{\bm{g}}(0, 0, \hat{\omega})\|^2 + 2L_g^2(\|\bm{y}(\bm{y}_0, s;\omega)\|^2+s^2)\Big)\text{d}s.
    \label{ineq_g}
\end{equation}
In Eq.~\eqref{ineq_g}, $\big(\bm{g}(0, 0, \omega), \hat{\bm{g}}(0, 0, \hat{\omega})\big)$ is seen as a random variables in $\mathbb{R}^{2n}$. For any coupling probability measure $\pi\big(\bm{g}(0, 0, \omega), \hat{\bm{g}}(0, 0, \hat{\omega})\big)$ such that its marginal distributions are $\eta_{0, 0}$ and $\hat{\eta}_{0, 0}$ ($\eta_{\bm{y}, s}$ and $\hat{\eta}_{\bm{y}, s}$ denote the probability measures of  $\bm{g}(\bm{y}, s, \omega)$ and $\hat{\bm{g}}(\bm{y}, s, \hat{\omega})$, respectively), we denote $\pi^*\big(\omega, \hat{\omega}\big)$ such that
\begin{equation}
    \pi^*\big(A, B) = \pi\Big(\big((\bm{g}(0, 0, \omega), \hat{\bm{g}}(0, 0, \hat{\omega})\big)(A, B)\Big),
\end{equation}
where $A\in\mathcal{B}(\Omega), B\in\mathcal{B}(\hat{\Omega})$. In other words, $pi$ is the pushforward measure of $\pi^*$. Here, $\big(\bm{g}(0, 0, \omega), \hat{\bm{g}}(0, 0, \hat{\omega})\big)$ is considered a measurable map from $\Omega\times \hat{\Omega}$ to $\mathbb{R}^n\times\mathbb{R}^n$.
 Specifically, if we take the expectation of Eq.~\eqref{ineq_g} and taking the infimum over all $\pi(\bm{g}(0, 0, \omega), \hat{\bm{g}}(0, 0, \hat{\omega}))$, we conclude that
\begin{equation}
\begin{aligned}
    \E_{(\omega, \hat{\omega})\sim\pi^*(\omega, \hat{\omega})}\big[\|\bm{y}(\bm{y}_0,t;\omega)-\hat{\bm{y}}(\bm{y}_0,t;\hat{\omega})\|^2\big]&\leq 6\exp(t + 2L_g^2t)\int_0^t W_2^2(\eta_{0, 0}, \hat{\eta}_{0, 0}) + 2L_g^2\Big(\E\big[\|\bm{y}(\bm{y}_0,s;\omega)\|^2\big] + s^2\Big)\text{d}s\\
    &\hspace{-3cm}\leq 12L_g^2\exp(t + 2L_gt)\int_0^t \Big(\E\big[\|\bm{y}(\bm{y}_0,s;\omega)\|^2\big] + s^2\Big)\text{d}s + 6t\exp(t + 2L_g^2t)\cdot\sup_{s, \bm{y}}W_2^2(\eta_{\bm{y}, s}, \hat{\eta}_{\bm{y}, s}).
\end{aligned}
    \label{ineq_g1}
\end{equation}
It is easy to verify that the marginal distributions of $\pi^*(\omega, \hat{\omega})$ are the distributions of $\omega$ and $\hat{\omega}$, respectively. We denote 
$\pi^{\sharp}(C, D),\,\, (C, D)\in\mathcal{B}(\mathbb{R}^n)\times\mathcal{B}(\mathbb{R}^n)$
to be the pushforward probability measure of $\pi^*$ such that 
\begin{equation}
\pi^{\sharp}(C, D) = \pi^*\Big(\big(\bm{y}(\bm{y}_0, t; \omega), \hat{\bm{y}}(\bm{y}_0, t; \omega)\big)^{-1}(C, D)\Big),
\end{equation}
where $\big(\bm{y}(\bm{y}_0, t, \omega), \hat{\bm{y}}(\bm{y}_0, t, \omega)\big)^{-1}(C, D)$ is the preimage of $(C, D)$ in $\mathcal{B}(\Omega)\times\mathcal{B}(\hat{\Omega})$.
Additionally, we can verify that the marginal distributions of $\pi^{\sharp}$ are $\mu_{\bm{y}_0, t}$ and $\hat{\mu}_{\bm{y}_0, t}$, respectively. Therefore, using the inequality~\eqref{ineq_g1}, the following inequality holds:
\begin{equation}
\begin{aligned}
    W_2^2(\mu_{\bm{y}_0, t}, \hat{\mu}_{\bm{y}_0, t})&\leq \E_{(\bm{y}, \hat{\bm{y}})\sim\pi^{\sharp}(\bm{y}, \hat{\bm{y}})}\Big[\big\|\bm{y}-\hat{\bm{y}}\big\|^2\Big] = \E_{(\omega, \hat{\omega})\sim\pi^{*}(\omega, \hat{\omega})}\Big[\big\|\bm{y}(\bm{y}_0, t;\omega)-\hat{\bm{y}}(\bm{y}_0, t;\hat{\omega})\big\|^2\Big]\\
    &\quad\leq 12L_g^2\exp(t + 2L_g^2t)\cdot\int_0^t \Big(\E\big[\|\bm{y}(\bm{y}_0,s;\omega)\|^2\big] + s^2\Big)\text{d}s + 6t\exp(1 + 2L_g^2t)\cdot\sup_{s, \bm{y}}W_2^2(\eta_{\bm{y}, s}, \hat{\eta}_{\bm{y}, s}),
\end{aligned}
\end{equation}
which completes the proof to Proposition~\ref{Propode}. The two constants $C_0,C_1$ in Proposition~\ref{Propode} are:
\begin{equation}
\begin{aligned}
   C_0\coloneqq  6t\exp(1 + 2L_g^2t),  \,\,\,\,  C_1\coloneqq 12L_g^2\exp(t + 2L_g^2t)\cdot\int_0^t \Big(\E\big[\|\bm{y}(\bm{y}_0,s;\omega)\|^2\big] + s^2\Big)\text{d}s 
\end{aligned}
\end{equation}
\end{proof}

Proposition~\ref{Propode} implies that minimizing 
\begin{equation}
\begin{aligned}
    \tilde{W}_{2,\delta}^{2, \text{e}}\big(\bm{y}(\bm{y}_0, t;\omega), \hat{\bm{y}}(\bm{y}_0, t;\hat{\omega})\big)&\approx \int_{\mathbb{R}^n} W_2^2(\mu_{\bm{y}_0, t}, \hat{\mu}_{\bm{y}_0, t})\text{d}\nu({\bm{y}}_0)\leq \max_{\bm{y}_0} W_2(\mu_{\bm{y}_0, t}, \hat{\mu}_{\bm{y}_0, t}) \cdot \nu(\mathbb{R}^n)\\
    &\quad \leq  \Big(C_0\sup_{\bm{y}, 0\leq s\leq t}W_2^2(\eta_{\bm{y}, s}, \hat{\eta}_{\bm{y}, s})+C_1\Big)\cdot \nu(\mathbb{R}^n),
\end{aligned}
\end{equation}
where $\nu({\bm{y}}_0)$ is the probability measure of the initial condition $\bm{y}_0$,
is necessary such that $W_2^2(\eta_{\bm{y}, s}, \hat{\eta}_{\bm{y}, s})$ is small for any $\bm{y}\in\mathbb{R}^n$ and $0\leq s\leq t$. In other words, 
if $\tilde{W}_{2,\delta}^{2, \text{e}}\big(\bm{y}(\bm{y}_0, t;\omega), \hat{\bm{y}}(\bm{y}_0, t;\hat{\omega})\big)$ is large, then there exists $\bm{y}, 0\leq s\leq t$ such that $W_2^2(\eta_{\bm{y}, s}, \hat{\eta}_{\bm{y}, s})$ is large and thus the distribution of $\bm{g}(\bm{y}, s, \omega)$ cannot be well approximated by the distribution of $\hat{\bm{g}}(\bm{y}, s, \hat{\omega})$.

\end{document}